\documentclass{article}

\usepackage[preprint, nonatbib]{neurips_2020}

\usepackage[utf8]{inputenc} % allow utf-8 input
\usepackage[T1]{fontenc}    % use 8-bit T1 fonts
\usepackage[hyphens]{url}% Attempt to make hyperref and algorithmic work together better:
\usepackage{hyperref}       % hyperlinks
\usepackage{algorithmic}
\usepackage{algorithm}
\usepackage{url}            % simple URL typesetting
\usepackage{booktabs}       % professional-quality tables
\usepackage{amsfonts}       % blackboard math symbols
\usepackage{nicefrac}       % compact symbols for 1/2, etc.
\usepackage{microtype}      % microtypography

\usepackage{graphicx}

\usepackage[english]{babel}
\usepackage[numbers,sort]{natbib}
\usepackage[colorinlistoftodos]{todonotes}
\renewcommand{\epsilon}{\varepsilon}

\usepackage{amsmath}

\usepackage{comment}
\usepackage{amsthm}
\usepackage{amssymb}
\usepackage{amsfonts}
\usepackage{dsfont}
\usepackage{thm-restate}
\usepackage{nicefrac}
\usepackage{footnote}
\usepackage{caption}
\usepackage{subcaption}
\usepackage{tikzsymbols}
\usepackage{tikz}
\usepackage{tikzscale}
\usepackage{textcomp}
\usetikzlibrary{patterns}

\usepackage{enumerate}
\usepackage{pgfplots}
\usepackage{adjustbox}
\usepgfplotslibrary{groupplots}
\usepackage{svg}

\newcommand{\iffullversion}[2]{%
#1%
}

\pgfplotsset{compat=1.16}

\usepackage[capitalize]{cleveref}

\DeclareMathOperator{\OPT}{OPT}

\DeclareMathOperator*{\argmin}{argmin}

\theoremstyle{plain}
\newtheorem{thm}{Theorem} %[section]
\newtheorem*{thm*}{Theorem}
\newtheorem{cor}[thm]{Corollary}

\newtheorem{lem}[thm]{Lemma}

\newtheorem{claim}[thm]{Claim}

\newtheorem*{cla*}{Claim}

\newtheorem{Def}[thm]{Definition}

\renewcommand{\ge}{\geqslant}
\renewcommand{\le}{\leqslant}
\renewcommand{\geq}{\geqslant}
\renewcommand{\leq}{\leqslant}

\newcommand{\opt}{\operatorname{OPT}}

\newcommand{\jobLife}{D}
\newcommand{\maxTime}{T}
\DeclareMathOperator{\predErr}{err}
\newcommand{\wReal}{w^{\mathrm{real}}}
\newcommand{\wPred}{w^{\mathrm{pred}}}

\newcommand{\robustalg}{\textsc{Robustify}\xspace}
\newcommand{\genrobustalg}{\textsc{General-Robustify}\xspace}

\renewcommand{\algorithmicrequire}{\textbf{Input:}}
\renewcommand{\algorithmicensure}{\textbf{Output:}}

\renewcommand{\algorithmicwhile}{\textbf{on arrival of}}
\renewcommand{\algorithmicendwhile}{\algorithmicend\ \textbf{on}}

\title{Learning Augmented Energy Minimization via Speed Scaling}

\author{%
  Etienne Bamas\thanks{Equal Contribution} \\
  EPFL \\
  Switzerland \\
  \texttt{etienne.bamas@epfl.ch} \\
  \And
  Andreas Maggiori$^*$ \\
  EPFL \\
  Switzerland \\
  \texttt{andreas.maggiori@epfl.ch} \\
  \And
  Lars Rohwedder$^*$ \\
  EPFL \\
  Switzerland \\
  \texttt{lars.rohwedder@epfl.ch} \\
  \And
  Ola Svensson$^*$ \\
  EPFL \\
  Switzerland \\
  \texttt{ola.svensson@epfl.ch}
}

\begin{document}

\maketitle

\begin{abstract}
As power management has become a primary concern in modern data centers, computing resources are being scaled dynamically to minimize energy consumption.
We initiate the study of a variant of the classic online \textit{speed scaling} problem, in which machine learning predictions about the future can be integrated naturally. Inspired by recent work on learning-augmented online algorithms, we propose an algorithm which incorporates predictions in a black-box manner and outperforms any online algorithm if the accuracy is high, yet maintains provable guarantees if the prediction is very inaccurate. We provide both theoretical and experimental evidence to support our claims.
\end{abstract}

\section{Introduction}
\label{sec:Introduction}
Online problems can be informally defined as problems where we are required to make irrevocable decisions without knowing the future. The classical way of dealing with such problems is to design algorithms which provide provable bounds on the ratio between the value of the algorithm's solution and the
optimal (offline) solution (the competitive ratio). Here, no assumption about the future is made.
Unfortunately, this \textit{no-assumption} regime comes at a high cost: Because the algorithm has to be overly prudent and prepare for all possible future events, the guarantees are often poor. Due to the success story of machine learning (ML), a recent line of work, first proposed by \citet{DBLP:conf/icml/LykourisV18} and \citet{DBLP:conf/nips/MedinaV17}, suggests incorporating the predictions provided by ML algorithms in the design of online algorithms. While some related approaches were considered before (see e.g.~\citet{DBLP:conf/isnn/XuX04}), the attention in this subject has increased substantially in the recent years~\cite{DBLP:conf/nips/PurohitSK18,
DBLP:conf/icml/LykourisV18,
DBLP:conf/nips/MedinaV17,
DBLP:conf/soda/LattanziLMV20,
DBLP:journals/corr/abs-1903-00092,
DBLP:journals/corr/abs-1911-07972,
DBLP:conf/icml/GollapudiP19,
IndykFrequencyestimation}.
An obvious caveat is that ML predictors often come with no worst-case guarantees and so we would like our algorithm to be robust to misleading predictions.
We follow the terminology introduced by \citet{DBLP:conf/nips/PurohitSK18},
where consistency is the performance of an algorithm when the predictor is perfectly accurate, while robustness is a worst case guarantee that does not depend on the quality of the prediction.
The goal of the works above is to design algorithms which provably beat the classical online algorithms
in the consistency case, while being robust when the predictor fails.
\paragraph{Problem.} The problem we are considering is motivated by the following
scenario. Consider a server that receives requests in an online fashion.
For each request some computational work has to be done and,
as a measure of Quality-of-Service, we require that each request is answered within some fixed time. In order to satisfy all the requests in time the server can dynamically change its processor speed at any time. However, the power consumption can be a super-linear function of the processing speed
(more precisely, we model the power consumption as $s^\alpha$ where $s$ is the processing speed and $\alpha>1$). Therefore, the problem of minimizing energy becomes non-trivial. This problem can be considered in the online model where the server has no information about the future tasks at all.
However, this assumption seems unnecessarily restrictive as these requests tend to
follow some patterns that can be predicted. For this reason a good algorithm should be able to incorporate some given predictions about the future.
%In this paper, we investigate the benefits of ML advice in the classic speed scaling problem  introduced by~\citet{DBLP:conf/focs/YaoDS95}. 
Similar scenarios appear in real-world systems as, for instance,
in dynamic frequency scaling of CPUs or in autoscaling of cloud applications~\cite{Azure-Autoscale, AWS-Autoscale}. In the case of autoscaling, ML advice is already being incorporated
into online algorithms in practice~\cite{AWS-Autoscale}. However, on the theory side, while the above speed scaling problem was introduced  by~\citet{DBLP:conf/focs/YaoDS95} in a seminal paper who studied it both in the online and offline settings (see also \cite{DBLP:journals/jacm/BansalKP07, DBLP:conf/latin/BansalBCP08}),
it has not been considered in the learning augmented setting.
\paragraph{Contributions.}
We formalize an intuitive and well-founded prediction model for the classic speed scaling problem.
We show that our problem is non-trivial by providing an unconditional lower bound
that demonstrates: An algorithm cannot be optimal, if the prediction is correct, and at the same time
retain robustness.
We then focus on our main contribution which is the design and analysis of a simple and efficient algorithm which incorporates any ML predictor as a black-box without making any further assumption.
We achieve this in a modular way:
First, we show that there is a consistent (but not robust)
online algorithm.
Then we develop a technique to make any online algorithm (which may use the prediction) robust at a small cost.
Moreover, we design general methods to allow algorithms
to cope with small perturbations in the prediction.
In addition to the theoretical analysis, we also provide an experimental analysis that supports our claims on both synthetic and real datasets.
For most of the paper we focus on a restricted case 
of the speed scaling problem by~\citet{DBLP:conf/focs/YaoDS95},
where predictions can be integrated naturally.
However, we show that with more sophisticated algorithms
our techniques extend well to the general case.
%On our way in proving these results we additionally prove: (1) a lower bound on the competitive
%ratio that any online algorithm can achieve in our problem and (2) a better competitive ratio than the one in the general setting, for a simple online algorithm proposed by~\citet{DBLP:conf/focs/YaoDS95}.

\paragraph{Related work.}
On the one hand, the field of learning augmented algorithms is relatively new, with a lot of recent exciting results (see for example~\citet{DBLP:conf/nips/PurohitSK18,
DBLP:conf/icml/LykourisV18,
DBLP:conf/nips/MedinaV17,
DBLP:conf/isnn/XuX04,
DBLP:conf/soda/LattanziLMV20,
IndykFrequencyestimation,
DBLP:journals/corr/abs-1903-00092,
DBLP:journals/corr/abs-1911-07972,
DBLP:conf/icml/GollapudiP19}). On the other hand, the speed scaling problem proposed by Yao et al. in~\cite{DBLP:conf/focs/YaoDS95} is well understood in both the offline and online setting. In its full generality, a set of tasks each with different arrival times, deadlines, and workloads needs to be completed in time while the speed is scaled in order to minimize energy. In the offline setting Yao et al. proved that the problem can be solved in polynomial time by a greedy algorithm. In the online setting, in which the jobs are revealed only at their release time, Yao et al. designed two different algorithms: (1) the \textsc{Average Rate} heuristic (AVR), for which they proved a bound of $2^{\alpha-1}\alpha^\alpha$ on the competitive ratio. This analysis was later proved to be asymptotically tight by~\citet{DBLP:conf/latin/BansalBCP08}. (2) The \textsc{Optimal Available} heuristic (OA), which was shown to be $\alpha^\alpha$-competitive in~\cite{DBLP:journals/jacm/BansalKP07}. In the same paper, Bansal et al. proposed a third online algorithm named BKP for which they proved a competitive ratio asymptotically equivalent to $e^\alpha$. While these competitive ratios exponential in $\alpha$ might not seem satisfying, Bansal et al. also proved that the exponential dependency cannot be better than $e^\alpha$. A number of variants of the problem have also been considered in the offline setting (no preemption allowed, precedence constraints, nested jobs and more listed in a recent survey by~\citet{DBLP:journals/scheduling/GerardsHH16}) and under a stochastic optimization point of view (see for instance \cite{stochastic_speed_scaling}). It is important to note that, while in theory the problem is interesting in the general case i.e. when $\alpha$ is an input parameter, in practice we usually focus on small values of $\alpha$ such as 2 or 3 since they model certain physical laws (see e.g. \citet{DBLP:journals/jacm/BansalKP07}). Although the BKP algorithm provides
the best asymptotic guarantee, OA or AVR often lead to better solutions 
for small $\alpha$ and therefore remain relevant.

\section{Model and Preliminaries}
\label{sec:Model and Preliminaries}

We define the Uniform Speed Scaling problem, a natural restricted version of the speed scaling problem~\cite{DBLP:conf/focs/YaoDS95}, where predictions can be integrated naturally. While the restricted version is our main focus as it allows for cleaner exposition and prediction models, we also show that our techniques can be adapted to more complex algorithms yielding similar results for  the general problem (see Section~\ref{sec:other} for further extensions).
%We then explain how it relates to the general problem for which our techniques also apply (with more complex  algorithms).

\paragraph{Problem definition.} An instance of the problem can be formally described as a triple $(w,\jobLife, \maxTime)$ where
$[0,\maxTime]$ is a finite time horizon,
each time $i\in\{0,\dotsc,\maxTime-\jobLife\}$ jobs with a total workload $w_i\in\mathbb Z_{\ge 0}$ arrive, which have
to be completed by time $i + \jobLife$.
To do so, we can adjust the speed $s_i(t)$
at which each workload $w_i$ is processed for $t\in \left[ i, i+\jobLife \right]$.
Jobs may be processed in parallel.
The overall speed of our processing unit at time $t$ is the sum $s(t) = \sum_i s_i(t)$, which yields
a power consumption of $s(t)^\alpha$,
where $\alpha >1$ is a problem specific constant.
Since we want to finish each job on time, we require that the amount of work dedicated to job $i$ in the interval $[i,i+\jobLife]$ should be $w_i$. In other words, $\int_{i}^{i+\jobLife} s_i(t) \ dt = w_i$.
In the offline setting, the whole instance is known in advance, i.e., the vector of workloads $w$ is entirely accessible. In the online problem, at time $i$, the algorithm is only aware of all workloads
$w_j$ with $j\le i$, i.e., the jobs that were released before time $i$. As noted by \citet{DBLP:journals/jacm/BansalKP07}, in the offline setting the problem can be formulated concisely as the following mathematical program:

\begin{Def}[Uniform Speed Scaling problem]{\rm
\label{convex program formulation}
On input $(w,\jobLife,\maxTime)$ compute the optimal solution for
\begin{equation*}
 \min \int_0^\maxTime {s(t)}^\alpha \ dt
 \quad s.t. \ \forall i \ \int_{i}^{i+\jobLife} s_i(t) \ dt = w_i,
 \quad \forall t \ \sum_i s_i(t)= s(t),
 \quad \forall i \forall t \ s_i(t) \geq 0 .
\end{equation*}}
\end{Def}
In contrast, we refer to the problem of~\citet{DBLP:conf/focs/YaoDS95} as the
\textit{General Speed Scaling} problem. The difference is
that there the time that the processor is given to complete each job is not necessarily equal across jobs. More precisely,
there we replace $w$ and $\jobLife$ by a set of jobs
$J_j = (r_j, d_j, w_j)$, where $r_j$ is the time the job
becomes available, $d_j$ is the deadline by which it
must be completed, and $w_j$ is the work to be completed. As a shorthand, we sometimes refer to these two problems as the \textit{uniform deadlines} case and the \textit{general deadlines} case. 
As mentioned before, \citet{DBLP:conf/focs/YaoDS95} provide a 
simple optimal greedy algorithm that runs in polynomial time.
As for the online setting, we emphasize that
both the general and the uniform speed scaling problem are non-trivial. More specifically, we prove that no online algorithm can have a competitive ratio better than $\Omega((6/5)^\alpha)$ even in the
uniform case (\iffullversion{see Theorem~\ref{thm:pure-online-lower-bound} in Appendix \ref{sec:purely_online}}{see full version of the paper}). We provide a few additional insights on the performance of online algorithms for the uniform deadline case. Although the AVR algorithm was proved to be $2^{\alpha-1}\cdot \alpha^\alpha$-competitive by~\citet{DBLP:conf/focs/YaoDS95} with a quite technical proof ; we show, with a simple proof, that AVR is in fact $2^\alpha$-competitive in the  uniform deadlines case and we provide an almost matching lower bound on the competitive ratio (\iffullversion{see Theorem~\ref{thm:AVR_ratio} and Theorem~\ref{thm:AVR_lowerbound} in appendix}{see full version of the paper}).

Note that in both problems the processor is allowed to run multiple jobs in parallel. However, we underline that restricting the problem to the case where the processor is only allowed to run at most one job at any given point in time is equivalent. Indeed, given a feasible solution $s(t)=\sum_i s_i(t)$ in the parallel setting, rescheduling jobs sequentially according to the earliest deadline first (EDF) policy creates a feasible solution of the same (energy) cost where at each point in time only one job is processed.

\paragraph{Prediction model and error measure.}
In the following, we present the model of prediction we are considering. Recall an instance of the problem is defined as a time horizon $[0, \maxTime]$, a duration $\jobLife$, and a vector of workloads
$w_i$, $i=1,\dotsc,\maxTime-\jobLife$.
A natural prediction is simply to give the algorithm a predicted instance $(\wPred, \maxTime, \jobLife)$ at time $t=0$. From now on, we will refer to the ground truth work vector as $\wReal$ and to the predicted instance as $\wPred$.We define the error $\predErr$ of the prediction as
\begin{equation*}\label{eq:error measure in the single prediction model}
    \predErr(\wReal, \wPred) = ||\wReal-\wPred||_\alpha^\alpha = \sum_{i} | \wReal_{i} - \wPred_{i} |^\alpha .
\end{equation*}
We simply write $\predErr$, when $\wReal$ and $\wPred$
are clear from the context.
The motivation for using $\alpha$ in the definition of $\predErr$ and not some other constant $p$ comes from strong impossibility results.
Clearly, guarantees for higher values $p$ are weaker than for lower $p$. Therefore, we would like to set $p$ as low as possible. However, we show that $p$ needs to be at least $\alpha$ in order to make a sensible use of a prediction (\iffullversion{see Theorem~\ref{thm:lower-bound-norm} in the supplementary material}{see full version of the paper}). 
We further note that it may seem natural to consider a predictor that is able to renew its prediction over time, e.g., by providing our algorithm a new prediction at every integral time $i$. To this end, \iffullversion{in Appendix \ref{sec:evolvingpredictors}}{in the full version of this paper}, we show how to naturally extend all our results from the single prediction to the evolving prediction model.
Finally we restate some desirable properties previously defined in~\cite{DBLP:conf/nips/PurohitSK18, DBLP:conf/icml/LykourisV18} that a learning augmented algorithm should have. Recall that the prediction is a source of unreliable information on the remaining instance and that the algorithm is oblivious to the quality of this prediction. In the following we denote by $\OPT$ the energy cost of the optimal offline schedule and by $\epsilon>0$ a robustness parameter of the algorithm, the smaller $\epsilon$ is the more we trust the prediction.

     If the prediction is perfectly accurate, i.e., the entire instance can be derived from
     the prediction,
     then the provable guarantees should be better
     than what a pure online algorithm can achieve.
     Ideally, the algorithm produces an offline optimal solution or comes close to it. By close to optimal, we mean that the cost of the algorithm (when the prediction is perfectly accurate) should be at most $c(\alpha,\epsilon) \cdot \OPT$, where $c(\alpha,\epsilon)$ tends to $1$ as $\epsilon$ approaches $0$. This characteristic will be called \textbf{consistency}.
     
     The competitive ratio of the algorithm should always be bounded
     even for arbitrarily bad (adversarial) predictions. Ideally, the competitive ratio is somewhat comparable to the competitive ratio
     of algorithms from literature for the pure online case.
     Formally, the cost of the algorithm should always be bounded by $r(\alpha,\epsilon) \cdot \OPT$ for some function $r(\alpha,\epsilon)$. This characteristic will be called \textbf{robustness}.
     
     A perfect prediction is a strong requirement.
     The consistency property should transition smoothly for all ranges of errors, that is, the algorithm's
     guarantees deteriorate smoothly as the prediction error
     increases. Formally, the cost of the algorithm should always be at most $c(\alpha,\epsilon) \cdot \OPT + f(\alpha,\epsilon,\predErr)$ for some function $f$ such that $f(\alpha,\epsilon,0)=0$ for any $\alpha,\epsilon$. This last property will be called \textbf{smoothness}.

Note that our definitions of consistency and robustness depend on the problem specific constant $\alpha$ which is unavoidable (\iffullversion{see Theorem~\ref{thm:pure-online-lower-bound} in the appendix}{see full version of the paper}).
The dependence on the robustness parameter $\epsilon$ is justified, because no algorithm can be perfectly consistent and robust at the same time (\iffullversion{see Theorem~\ref{thm:robust_consistent_tradeoff}
in the appendix}{see full version}), hence a trade-off is necessary.

\newcommand{\wOnline}[0]{\ensuremath{w^{\mathrm{online}}}\xspace}

\section{Algorithm}
\label{sec:algorithm}
In this section we develop two modular building blocks
to obtain a consistent, smooth, and robust algorithm.
The first block is an algorithm which computes a schedule online
taking into account the prediction for the future.
This algorithm is consistent and smooth, but not robust.
Then we describe a generic method how to robustify an arbitrary online
algorithm at a small cost.
Finally, we give a summary of the theoretical qualities for the
full algorithm and a full description in pseudo-code.
We note that \iffullversion{in
Appendix \ref{sec:general_deadlines} and Appendix \ref{sec:noise_tolerance}}{in the full version of the paper} we present additional building blocks (see Section~\ref{sec:other} for an overview).

\subsection{A Consistent and Smooth Algorithm}
In the following we describe a learning-augmented online algorithm, which
we call \textsc{LAS-Trust}.

\paragraph*{Preparation.}
We compute an optimal schedule $s^\text{pred}$ for
the predicted jobs.
An optimal schedule can always be normalized
such
that each workload $\wPred_i$ is completely
scheduled in an interval $[a_i, b_i]$
at a uniform speed $c_i$, that is,
\begin{equation*}
    s^\text{pred}_i(t) = \begin{cases}
      c_i &\text{ if } t\in [a_i, b_i], \\
      0 &\text{ otherwise.}
    \end{cases}
\end{equation*}
Furthermore, the intervals $[a_i, b_i]$ are
non-overlapping.
For details we refer the reader to
the optimal offline algorithm by~\citet{DBLP:conf/focs/YaoDS95},
which always creates such a schedule.
\paragraph*{The online algorithm.}
At time $i$ we first
schedule $\wReal_i$ at uniform speed in $[a_i, b_i]$,
but we cap the speed at $c_i$. If 
this does not complete the job,
that is, $\wReal_i > c_i (b_i - a_i) = \wPred_i$,
we uniformly schedule the remaining work
in the interval $[i, i + \jobLife]$

More formally,we define
$s_i(t) = s'_i(t) + s''_i(t)$ where
\begin{equation*}
    s'_i(t) = \begin{cases}
      \min\left\{\frac{\wReal_i}{b_i - a_i}, c_i\right\} &\text{if } t\in [a_i, b_i], \\
      0 &\text{otherwise.}
    \end{cases}
\end{equation*}
and
\begin{equation*}
    s''_i(t) = \begin{cases}
      \frac{1}{\jobLife} \max\{0, \wReal_i - \wPred_i\} &\text{if } t\in [i, i + \jobLife], \\
      0 &\text{otherwise.}
    \end{cases}
\end{equation*}
\paragraph*{Analysis.}
It is easy to see that the algorithm is consistent:
If the prediction of $\wReal_i$ is perfect
($\wPred_i = \wReal_i$), the job
will be scheduled at speed $c_i$ in the
interval $[a_i, b_i]$. If all predictions
are perfect, this is exactly the optimal schedule.

\begin{thm}\label{thm:cost-trust}
  For every $0 < \delta \le 1$, the cost of the schedule produced by the algorithm \textsc{LAS-Trust} is bounded by
 $% \begin{equation*}
      (1 + \delta)^\alpha\OPT + (12/\delta)^\alpha \cdot \predErr.
  $%\end{equation*}
\end{thm}
\begin{proof}
  Define $w^+_i = \max\{0, \wReal_i - \wPred_i\}$
  as the additional work at time $i$ as compared
  to the predicted work.
  Likewise, define $w^-_i = \max\{0, \wPred_i - \wReal_i\}$.
  We use $\OPT(w^+)$ and $\OPT(w^-)$ to denote the cost of optimal schedules of these workloads $w^+$ and $w^-$, respectively.
  We will first relate the energy of the
  schedule $s(t)$ to the optimal energy
  for the predicted instance, i.e.,
  $\OPT(\wPred)$.
  Then we will relate
  $\OPT(\wPred)$ to $\OPT(\wReal)$.
  
  For the former let $s'_i$ and $s''_i$ be
  defined as in the algorithm. Observe that
  $s'_i(t) \le s^\text{pred}_i(t)$ for all $i$ and $t$. Hence, the energy for the partial schedule $s'$ (by itself) is at most $\OPT(\wPred)$.
  Furthermore,
  by definition we have that $s''_i(t) = w^+_i / D$.
  In other words, $s''_i$ is exactly
  the \textsc{AVR} schedule on instance $w^+$.
  By analysis of \textsc{AVR}, we know
  that the total energy of $s''_i$ is
  at most $2^\alpha \OPT(w^+)$.
  Since the energy function is non-linear,
  we cannot simply add the energy of both speeds.
  Instead, we use the following inequality:
  For all $x, y\ge 0$ and $0 < \gamma \le 1$, it holds that
$%$\begin{equation*}
    (x + y)^\alpha \le (1 + \gamma)^\alpha x^\alpha + \left(\frac{2}{\gamma}\right)^\alpha y^\alpha .
$%\end{equation*}
This follows from a simple case distinction whether $y \le \gamma x$.
Thus, (substituting $\gamma$ for $\delta / 3$)
the energy of the schedule $s$ is bounded by
\begin{align}
    \int (s'(t) + s''(t))^\alpha dt \notag 
    &\le (1 + \delta / 3)^\alpha \int s'_i(t)^\alpha dt
    +  (6/ \delta)^\alpha \int s''_i(t)^\alpha dt \notag \\
    &\le (1 + \delta / 3)^\alpha \OPT(\wPred) + (12 / \delta)^\alpha \OPT(w^+) .
    \label{smooth-schedule-1}
\end{align}
For the last inequality we used that  the competitive ratio of
\textsc{AVR} is $2^\alpha$.
%is absorbed by the term
%$O(1 / \delta)^\alpha$.

In order to relate $\OPT(\wPred)$ and
$\OPT(\wReal)$, we argue similarly.
Notice that scheduling $\wReal$ optimally (by itself)
and then scheduling $w^-$ using \textsc{AVR}
forms a valid solution for $\wPred$.
Hence,
\begin{align*}
    \OPT(\wPred) \le (1 + \delta / 3)^\alpha \OPT(\wReal) + (12 / \delta)^\alpha \OPT(w^-) .
\end{align*}
Inserting this inequality into (\ref{smooth-schedule-1}) we conclude that
the energy of the schedule $s$ is at most
\begin{align*}
    & (1 + \delta / 3)^{2\alpha} \OPT(\wReal) + (12/\delta)^\alpha (\OPT(w^+) + \OPT(w^-)) \\
    &\le  (1 + \delta)^\alpha \OPT(\wReal) + (12/\delta)^\alpha \cdot \predErr .
\end{align*}
This inequality follows from the fact that
the error function $\lVert\cdot\rVert_\alpha^\alpha$ is always an upper bound on the energy of the optimal schedule
(by scheduling every job within the next time unit). 
\end{proof}

\subsection{Robustification}

In this section, we describe a method
\textsc{Robustify} that takes any
online algorithm
which guarantees to complete each job
in $(1 - \delta)\jobLife$ time, that is,
with some slack to its deadline, and
turns it into a robust algorithm without
increasing the energy of the
schedule produced.
Here $\delta > 0$ can be chosen at will, but it impacts the robustness guarantee.
We remark that the slack constraint 
is easy to achieve:
\iffullversion{In Appendix \ref{sec:a_shrinking_lemma}}{In the full version
of the paper} we prove that decreasing $\jobLife$
to $(1 - \delta)\jobLife$ increases
the energy of the optimum schedule
only very mildly. Specifically, if we let $\OPT(\wReal, (1-\delta)\jobLife, T)$ and $\OPT(\wReal, \jobLife, T)$ denote the costs of optimal schedules of workload $\wReal$ with durations $(1-\delta)\jobLife$ and $\jobLife$, respectively, then:  
\begin{claim}\label{cla:shrink}
  For any instance $(\wReal, \jobLife, \maxTime)$
  we have that
$%$  \begin{align*}
      \OPT(\wReal, (1 - \delta)\jobLife, T)
      \le \frac{1}{(1 - \delta)^{\alpha-1}} \OPT(\wReal, \jobLife, T) .
 $% \end{align*}
\end{claim}

Hence, running a consistent algorithm with $(1 - \delta)\jobLife$
will not increase the cost significantly.
Alternatively, we can run the online
algorithm with $\jobLife$, but increase the
generated speed function by $1/(1 - \delta)$ and reschedule all jobs using EDF. This also results
in a schedule where all jobs are completed
in $(1 - \delta)\jobLife$ time.

For a schedule $s$ of
$(\wReal, (1 - \delta)\jobLife, \maxTime)$
we define the $\delta$-convolution operator which returns the schedule $s^{(\delta)}$ of the original instance $(\wReal, \jobLife, \maxTime)$  by 
\begin{equation*}
s_i^{(\delta)}(t)=\frac{1}{\delta \jobLife} \int_{t-\delta \jobLife}^{t} s_i(r) \ dr
\end{equation*}
for each $i\in T$ (letting $s_i(r) = 0$ if $r<0$). 
See Figure~\ref{fig:convolution} for
an illustration.
The name comes from the fact that this operator is the convolution of $s_i(t)$ with the function $f(t)$ that takes value $1/(\delta \jobLife)$ if $0 \leq t \leq \delta \jobLife$ and value $0$ otherwise.
%\begin{equation*}
%    f(t) = \begin{cases}
%      \frac{1}{\delta \jobLife} &\text{if $0 \le t \le \delta \jobLife$}, \\
%      0 &\text{otherwise.}
%    \end{cases}
%\end{equation*}

Next we state three key properties
of the convolution operator, all of
which follow from easy observations or
standard arguments that are deferred to
\iffullversion{Appendix \ref{sec:appendix_convolution}}{the full version of the paper}.
\begin{restatable}{claim}{CONVcla}
\label{cla:convolution1}
If $s$ is a feasible schedule for $(\wReal, (1 - \delta)\jobLife, \maxTime)$ then $s^{(\delta)}$ is a feasible schedule for $(\wReal, \jobLife, \maxTime)$.
\end{restatable}

\begin{restatable}{claim}{CONVclai}
\label{cla:convolution_energy}
The cost of schedule
$s^{(\delta)}$ is not higher
than that of $s$, that is,
\begin{equation*}
    \int_0^\maxTime (s^{(\delta)}(t))^\alpha dt
    \le 
    \int_0^\maxTime (s(t))^\alpha dt .
\end{equation*}
\end{restatable}

Let $s_i^{\mathrm{AVR}}(t)$ denote
the speed of workload $\wReal_i$ of the \textsc{Average Rate}
heuristic, that is,
$s_i^{\mathrm{AVR}}(t) = \wReal_i / \jobLife$
if $i \le t \le i + \jobLife$ and
$s_i^{\mathrm{AVR}}(t) = 0$ otherwise.
%Since \textsc{Average Rate} is known to
%be a constant-competitive online algorithm, it
%suffices to compare 
We relate $s^{(\delta)}_i(t)$ to
$s_i^{\mathrm{AVR}}(t)$.

\begin{restatable}{claim}{CONVclaim}
\label{cla:convolution_robust}
  Let $s$ be a feasible schedule for
  $(\wReal, (1 - \delta)\jobLife, \maxTime)$. Then
  $%\begin{equation*}
      s^{(\delta)}_i(t) \le \frac{1}{\delta} s^{\mathrm{AVR}}_i(t).
  $%\end{equation*}
\end{restatable}

By using that the competitive ratio
of \textsc{Average Rate} is at most $2^\alpha$ (see \iffullversion{Appendix \ref{sec:purely_online}}{the full version of this paper}), we get
\begin{equation*}
      \int_0^\maxTime (s^{(\delta)}(t))^\alpha dt \le \left(\frac{1}{\delta}\right)^\alpha
      \int_0^\maxTime 
      (s^{\mathrm{AVR}}(t))^\alpha dt\le \left(\frac{2}{\delta}\right)^\alpha
      \OPT .
\end{equation*}

\begin{figure}
    \centering
    \includegraphics[scale=0.45]{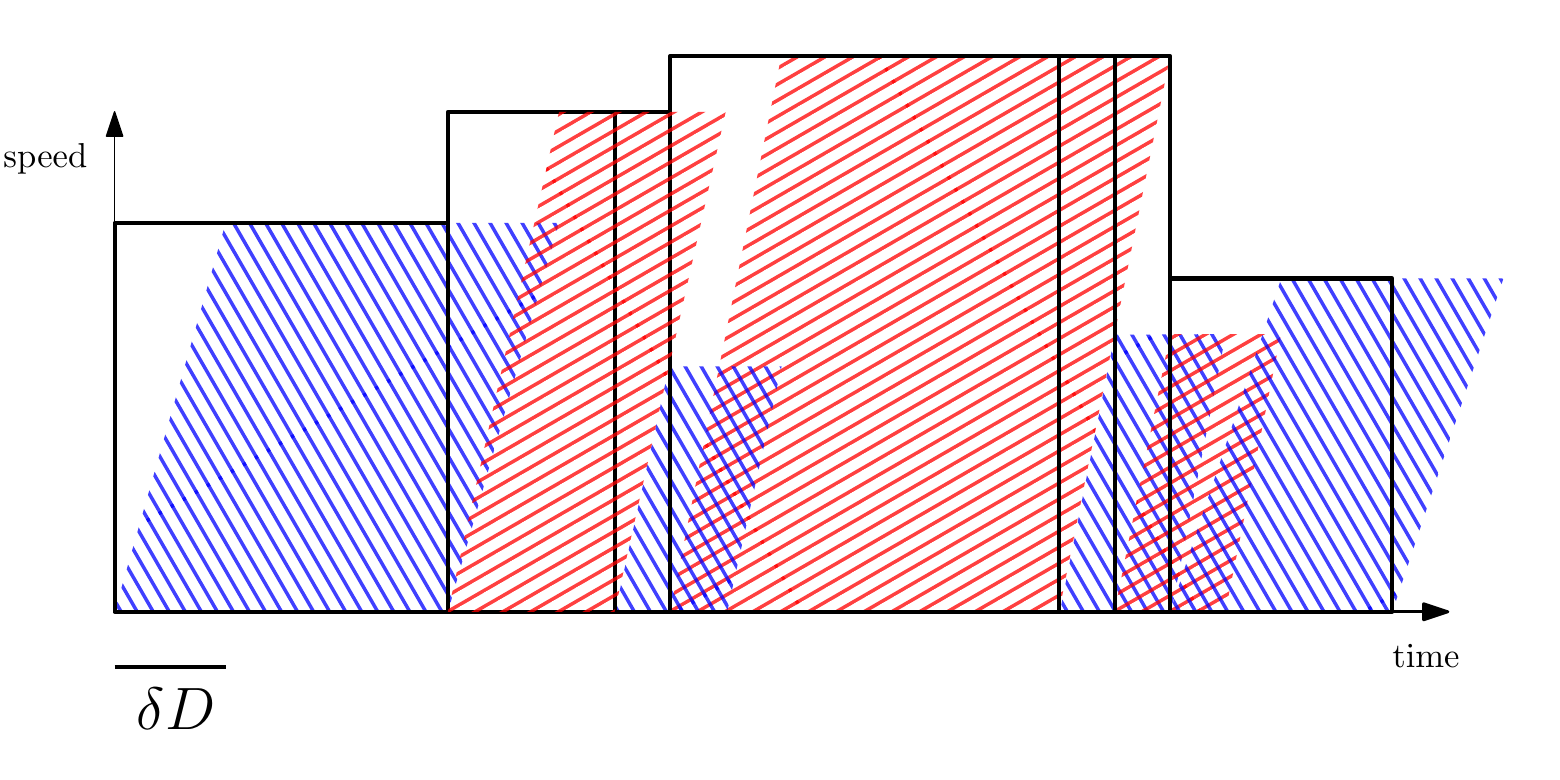}
    \caption{A schedule and its convolution.}
    \label{fig:convolution}
\end{figure}
We conclude with the following theorem,
which follows immediately from the
previous claims.
\begin{thm}\label{thm:cost-robustify}
  Given an online algorithm that
  produces a schedule $s$ for
  $(\wReal, (1 - \delta)\jobLife, \maxTime)$, we can compute online
  a schedule $s^{(\delta)}$ with
  \begin{equation*}
    \int_0^{\maxTime} (s^{(\delta)}(t))^\alpha dt
    \le 
    \min \left\lbrace 
    \int_0^{\maxTime} (s(t))^\alpha dt,\ 
    \left(\frac{2}{\delta}\right)^\alpha  \OPT  \right\rbrace .
  \end{equation*}
\end{thm}

\subsection{Summary of the Algorithm}

\begin{algorithm}[tb]
   \caption{\textsc{Learning Augmented Scheduling (LAS)}}
   \label{alg:LAS}
\begin{algorithmic}
   \STATE {\bfseries Input:} $\maxTime$, $\jobLife$, and $\wPred$ initially and $\wReal$ in an online fashion
   \STATE {\bfseries Output:} A feasible schedule $(s_i)^{\maxTime-\jobLife}_{i=0}$
   \STATE Let $\delta > 0$ with $\big(\frac{1 + \delta}{1 - \delta}\big)^\alpha = 1 + \epsilon$.
   \STATE Compute optimal offline schedule for $(\wPred, T, (1 - \delta)D)$
     where the jobs $\wPred_i$ are run at uniform speeds $c_i$ an disjoint intervals $[a_i, b_i]$ using~\cite{DBLP:conf/focs/YaoDS95}.
   \\[1em]
   \WHILE{$\wReal_i$}
        \STATE Let $s'_i(t) = \begin{cases}
      \min\left\{\frac{\wReal_i}{b_i - a_i}, c_i\right\} &\text{if } t\in [a_i, b_i], \\
      0 &\text{otherwise.}
    \end{cases}$
        \STATE Let $s''_i(t) = \begin{cases}
      \frac{1}{\jobLife} \max\{0, \wReal_i - \wPred_i\} &\text{if } t\in [i, i + \jobLife], \\
      0 &\text{otherwise.}
    \end{cases}$
        \STATE Let $s_i(t) = \frac{1}{\delta\jobLife}\int_{t-\delta\jobLife}^t s'_i(r) + s''_i(r) \ dr$
   \ENDWHILE
\end{algorithmic}
\end{algorithm}

By combining \textsc{LAS-Trust} and \textsc{Robustify}, we obtain
an algorithm \textsc{LAS} (see Algorithm~\ref{alg:LAS}) which has the following properties. See \iffullversion{Appendix~\ref{omitted_proof_secthree}}{the full version of this paper} for a formal argument.
\begin{restatable}{thm}{LAS} 
 \label{thm:main_algorithm}
  For any given $\epsilon > 0$,
  algorithm
  \textsc{LAS} constructs
  a schedule of cost at most %such that
  $%\begin{equation*}
      %\int_0^{\maxTime} (s_i(t))^\alpha dt
      %\le 
      \min \left\{(1 + \epsilon)\OPT + O\left(\frac{\alpha}{\epsilon}\right)^\alpha
      \predErr, \ 
      O\left(\frac{\alpha}{\epsilon}\right)^\alpha \OPT \right\} .
  $%\end{equation*}
\end{restatable}

\subsection{Other Extensions}
\label{sec:other}

In \iffullversion{Appendix~\ref{sec:general_deadlines}}{the full version} we also consider General Speed Scheduling (the problem with general deadlines) and show that a more sophisticated method allows us to robustify
any algorithm even in this more general setting.
Hence, for this case we can also obtain an
algorithm that is almost optimal in the consistency case and always robust.

The careful reader may have noted that one can craft instances so that the used error function $\predErr$ is very sensitive to small shifts in the prediction.  
An illustrative example is as follows. Consider a predicted workload $\wPred$  defined by $\wPred_i = 1$ for those time steps $i$ that are divisible by a large constant, say $1000$, and let $\wPred_i =0$ for all other time steps. 
If the real instance $\wReal$ is a small shift  of $\wPred$ say $\wReal_{i+1}= \wPred_i$ then the prediction error $\predErr(\wReal, \wPred)$ is large although $\wPred$ intuitively forms a good prediction of $\wReal$. 
To overcome this sensitivity, we first generalize  the definition of $\predErr$ to $\predErr_\eta$ which is tolerant to small shifts in the workload. In particular, $\predErr_\eta(\wReal, \wPred) =0$ for the example given above.  We then give a generic method for transforming an algorithm so as to obtain guarantees with respect to $\predErr_\eta$ instead of $\predErr$ at a small loss. Details can be found in 
\iffullversion{Appendix~\ref{sec:noise_tolerance}.}
{the full version of the paper.}

\section{Experimental analysis}
\label{sec:Experimental analysis}
\begin{figure*}[!htb]
\centering
\centerline{\includegraphics[width=0.9\textwidth]{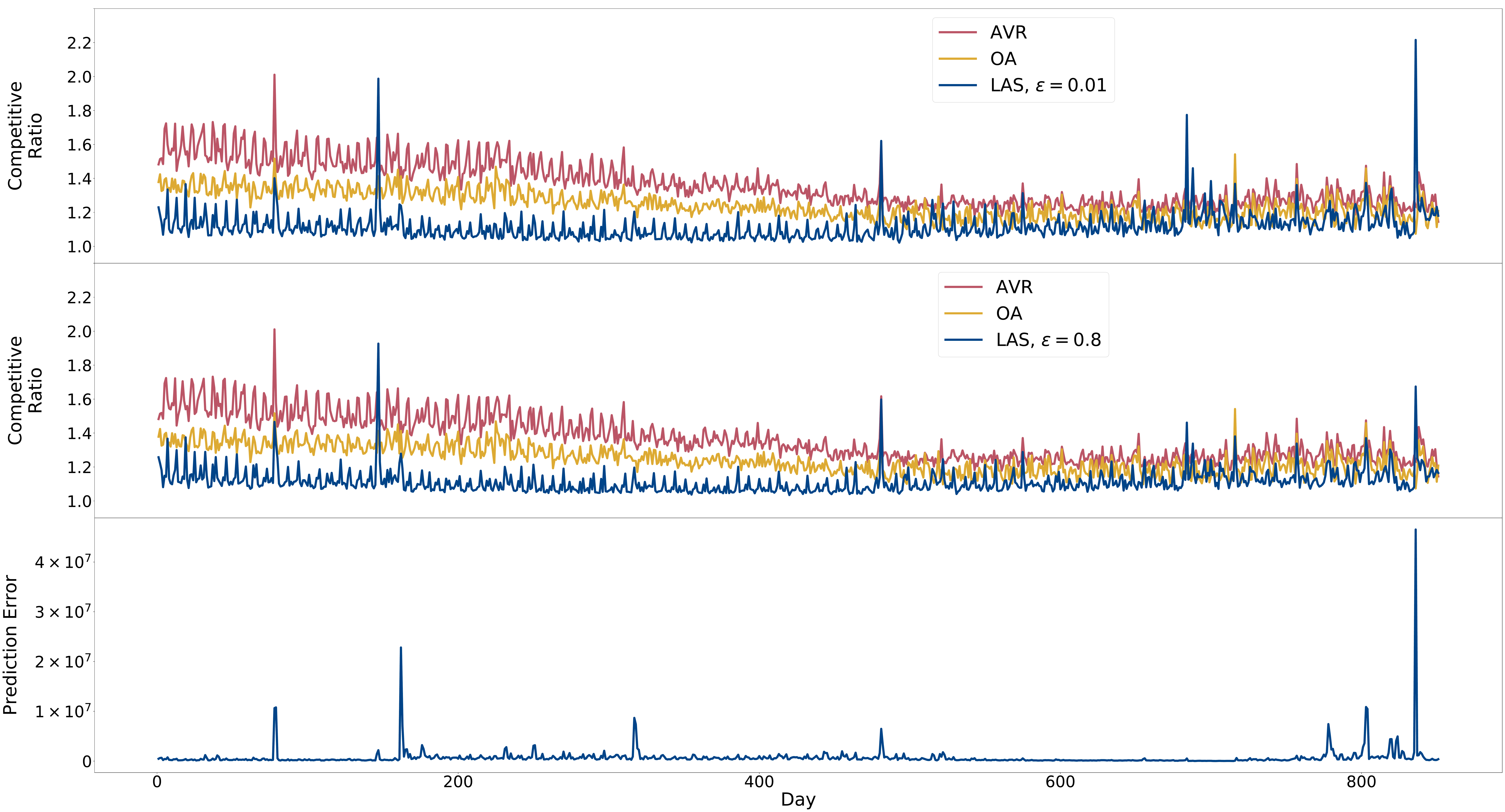}}
    \vskip 0.05in
\caption{From top to bottom: The first two graphs show the performance of LAS when $\epsilon = 0.01$ and $\epsilon = 0.8$ with respect to the online algorithms \textsc{AVR} and \textsc{OA}. The bottom graph presents the prediction error. The timeline was discretized in chunks of ten minutes and $D$ was set to 20.}
    \label{fig:realdataset}
    \vskip 0.05in
\end{figure*}

In this section, we will test the LAS algorithm on both synthetic and real datasets. We will calculate the competitive ratios with respect to the offline optimum.
We fix $\alpha = 3$ in all our experiments as this value models the power consumption of modern processors (see~\citet{DBLP:journals/jacm/BansalKP07}). For each experiment, we compare our LAS algorithm to the three main online algorithms that exist for this problem which are AVR and OA by~\citet{DBLP:conf/focs/YaoDS95} and BKP by~\citet{DBLP:journals/jacm/BansalKP07}. We note that the code is publicly available at \url{https://github.com/andreasr27/LAS}.

\paragraph{Artificial datasets.} In the synthetic data case, we will mimic the request pattern of a typical data center application by simulating a bounded random walk. In the following we write $Z\sim \mathcal{U}\{m,M \}$ when sampling an integer uniformly at random in the range $[m,M]$. Subsequently, we fix three integers $s,m,M$ where $[m,M]$ define the range in which the walk should stay.
For each integral time $i$ we sample $X_i\sim \mathcal{U}\{-s,s\}$. Then we set $w_{0}\sim \mathcal{U}\{m,M\}$ and $w_{i+1}$ to be the median value of the list $\{m,w_i+X_i,M\}$, that is, if the value $w_i+X_i$ remains in the predefined range we do not change it, otherwise we round it to the closest point in the range. For this type of ground truth instance we test our algorithm coupled with three different predictors. The \textbf{accurate} predictor for which we set $\Tilde w_i \sim w_i+\mathcal{U}\{-s,s \}$, the \textbf{random} predictor where we set $\Tilde w_i \sim \mathcal{U}\{m,M \}$ and the \textbf{misleading} predictor for which $\Tilde w_i = (M-w_i) + m$. In each case we perform 20 experiment runs.
The results are summarized in Table~\ref{table:normal regime results}.
In the first two cases (accurate and random predictors) we present the average competitive ratios of every algorithm over all runs.
In contrast, for the last column (misleading predictor)
we present the maximum competitive ratio of each algorithm taken over the 20 runs
to highlight the worst case robustness of LAS.
We note that in the first case, where the predictor is relatively accurate but still noisy, LAS is consistently better than any online algorithm achieving a competitive ratio close to $1$ for small values of $\epsilon$. In the second case, the predictor does not give us useful information about the future since it is completely uncorrelated with the ground truth instance. In such a case, LAS experiences a similar performance to the best online algorithms. In the third case, the predictor tries to mislead our algorithm by creating a prediction which constitutes a symmetric (around $(m+M)/2$) random walk with respect to the true instance. When coupled with such a predictor, as expected, LAS performs worse than the best online algorithm, but it still maintains an acceptable competitive ratio.
Furthermore, augmenting the robustness parameter $\epsilon$, and thereby
trusting less the predictor, improves the competitive ratio in this case.

\begin{table}
\centering
\caption{Artificial dataset results}
\vskip 0.1in
\label{table:normal regime results}
{\small
\begin{tabular}{lcccr}
\toprule
Algorithm & Accurate & Random & Misleading \\
\midrule
AVR     & 1.268 & 1.268 & 1.383 \\
BKP     & 7.880     & 7.880     & 10.380   \\
OA     & 1.199     & 1.199     & 1.361   \\
LAS, $\epsilon = 0.8$    & 1.026& 1.203& 1.750\\
LAS, $\epsilon = 0.6$    & 1.022& 1.207& 1.758\\
LAS, $\epsilon = 0.4$    & 1.018& 1.213& 1.767\\
LAS, $\epsilon = 0.2$    & 1.013& 1.224& 1.769\\
LAS, $\epsilon = 0.01$    & 1.008& 1.239& 1.766\\
\bottomrule
\end{tabular}
}
\caption*{We used $m=20$, $M=80$, $s=5$, $T =220$ and $D = 20$.}
\vskip -0.2in
\end{table}

\paragraph{Real dataset.}
We provide additional evidence that the \textsc{LAS} algorithm outperforms purely online algorithms by conducting experiments on the login requests to \textit{BrightKite} \cite{Brightkite}, a no longer functioning social network. We note that this dataset was previously used in the context of learning augmented algorithms by~\citet{DBLP:conf/icml/LykourisV18}. In order to emphasize the fact that even a very simple predictor can improve the scheduling performance drastically, we will use the arguably most simple predictor possible. We use the access patterns of the previous day as a prediction for the current day. In Figure~\ref{fig:realdataset} we compare the performance of the \textsc{LAS} algorithm for different values of the robustness parameter $\epsilon$ with respect to \textsc{AVR} and \textsc{OA}. We did not include \textsc{BKP}, since its performance is substantially worse than all other algorithms. Note that our algorithm shows a substantial improvement with respect to both \textsc{AVR} and \textsc{OA}, while maintaining a low competitive ratio even when the prediction error is high (for instance in the last days). The first $100$ days, where the prediction error is low, by setting $\epsilon = 0.01$ (and trusting more the prediction) we obtain an average competitive ratio of ${1.134}$, while with $\epsilon = 0.8$ the average competitive ratio slightly deteriorates to $1.146$. However, when the prediction error is high, setting $\epsilon = 0.8$ is better. On average from the first to the last day of the timeline, the competitive ratio of AVR and OA is $1.36$ and $1.24$ respectively, while LAS obtains an average competitive ratio of $1.116$ when $\epsilon = 0.01$ and $1.113$ when $\epsilon=0.8$, thus beating the online algorithms in both cases.

More experiments regarding the influence of the $\alpha$ parameter in the performance of \textsc{LAS} algorithm can be found in \iffullversion{Appendix~\ref{sec:Additional Experiments}.}
{the full version of the paper.}

\section*{Broader impact} As climate change is a severe issue, trying to minimize the environmental impact of modern computer systems has become a priority. High energy consumption and the CO\textsubscript{2} emissions related to it are some of the main factors increasing the environmental impact of computer systems. While our work considers a specific problem related to scheduling, we would like to emphasize that a considerable percentage of real-world systems already have the ability to dynamically scale their computing resources\footnote{CPU Dynamic Voltage and Frequency Scaling (DVFS) in modern processors and autoscaling of cloud applications} to minimize their energy consumption. Thus, studying models (like the one presented in this paper) with the latter capability is a line of work with huge potential societal impact. In addition to that, although the analysis of the guarantees provided by our algorithm is not straightforward, the algorithm itself is relatively simple. The latter fact makes us optimistic that insights from this work can be used in practice contributing to minimizing the environmental impact of computer infrastructures.

\section*{Acknowledgments and Disclosure of Funding}
This research is supported by the Swiss National Science Foundation project 200021-184656 ``Randomness in Problem Instances and Randomized Algorithms''. Andreas Maggiori was supported by the Swiss National Science Fund (SNSF) grant n\textsuperscript{o} $200020\_182517/1$ ``Spatial Coupling of Graphical Models in Communications, Signal Processing, Computer Science and Statistical Physics''.

\bibliographystyle{plainnat}
\bibliography{main}

\iffullversion{
\clearpage
% This is to separate the appendix from the main paper. In the submission we
% have to make two independent files 

\appendix
\section{Omitted Proofs from Section~\ref{sec:algorithm}}
\label{omitted_proof_secthree}

\LAS*
\begin{proof}
We choose $\delta$ such that
$(\frac{1 + \delta}{1 - \delta})^\alpha = 1 + \epsilon$. Note that $\delta \le \epsilon/(6\alpha)$.
By Claim~\ref{cla:shrink} we know
that
\begin{equation*}
    \OPT(\wReal, (1 - \delta)\jobLife, \maxTime)
    \le \left(\frac{1}{1 - \delta}\right)^\alpha \OPT .
\end{equation*}
Hence, by Theorem~\ref{thm:cost-trust} algorithm \textsc{LAS-Trust} constructs
a schedule with cost at most
\begin{equation*}
    \left(\frac{1 + \delta}{1 - \delta}\right)^\alpha \OPT + O\left(\frac{1}{\delta}\right)^\alpha \predErr
\end{equation*}
Finally, we apply \textsc{Robustify} and
with Theorem~\ref{thm:cost-robustify} obtain a bound of 
\begin{multline*}
    \min\left\{
    \left(\frac{1 + \delta}{1 - \delta}\right)^\alpha \OPT + O\left(\frac{1}{\delta}\right)^\alpha \predErr,\ 
    O\left(\frac{1}{\delta}\right)^\alpha \OPT \right\} \\
    \le
    \min\left\{
    (1 + \epsilon) \OPT + O\left(\frac{\alpha}{\epsilon}\right)^\alpha \predErr,\ 
    O\left(\frac{\alpha}{\epsilon}\right)^\alpha \OPT \right\} . \qedhere
\end{multline*}
\end{proof}

\section{Pure online algorithms for uniform deadlines}
\label{sec:purely_online}
Since most related results concern the general speed scaling problem, we
give some insights about the uniform speed scaling problem in the online setting without predictions. We first give a lower bound on the competitive ratio for any online algorithm for the simplest case where $D=2$ and then provide an almost tight analysis of the competitive ratio of \textsc{AVR}.
\begin{thm}\label{thm:pure-online-lower-bound}
There is no (randomized) online algorithm with an (expected) competitive ratio better
than $\Omega \left(\left(6/5\right)^\alpha\right)$.
\end{thm}
\begin{proof}
Consider $D = 2$ and two instances $J_1$ and $J_2$. Instance $J_1$ consists of only one job that is released at time $0$ with workload $1$ and $J_2$ consists of the same first job with a second job which starts at time $1$ with workload $2$.

In both instances, the optimal schedule runs with uniform speed at all time.
In the first instance, it runs the single job for $2$ units of time at speed $1/2$.
The energy-cost is therefore $1/2^{\alpha-1}$.
In the second instance, it first runs the first job at speed $1$ for one unit of time
and then the second job at speed $1$ for $2$ units of time.
Hence, it has an energy-cost of $3$.

Now consider an online algorithm. Before time $1$ both instances are identical and the
algorithm therefore behaves the same. In particular, it has to decide how much work of job $1$
to process between time $0$ and $1$. Let us fix some $\gamma\geq 0$ as a threshold for the amount of work dedicated to job $1$ by the algorithm before time $1$. We have the following two cases depending on the instance.
\begin{enumerate}
    \item If the algorithm processes more that $\gamma$ units of work on job $1$ before time $1$ then for instance $J_1$ the energy cost is at least $\gamma^\alpha$. Hence the competitive ratio is at least
    $\gamma^\alpha \cdot 2^{\alpha-1}$.
    \item On the contrary, if the algorithm works less than $\gamma$ units of work before the release of the second job then in instance $J_2$ the algorithm has to complete at least $3-\gamma$ units of work between time $1$ and $3$. Hence, its competitive ratio is at least $2/3 \cdot ((3-\gamma) / 2)^\alpha$.
\end{enumerate}

Choosing $\gamma$ such that these two competitive ratios are equal gives $\gamma=\frac{3}{3^{1/\alpha}4^{1-1/\alpha}+1}$ and yields a lower bound on the competitive ratio of at least:
\begin{equation*}
    2^{\alpha-1} \left(\frac{3}{3^{1/\alpha} 4^{1-1/\alpha}+1}\right)^\alpha .
\end{equation*}
This term asymptotically approaches $1/2 \cdot (6/5)^\alpha$ and this already proves the theorem for deterministic algorithms. More precisely, it proves that any deterministic algorithm has a competitive ratio of at least $\Omega \left((6/5)^\alpha\right)$ on at least one of the two instances $J_1$ or $J_2$.
Hence, by defining a probability distribution over inputs such that $p(J_1)=p(J_2)=\frac{1}{2}$ and applying Yao's minimax principle we get that the expected competitive ratio of any randomized online algorithm is at least 

\begin{equation*}
    (1/2)\cdot 2^{\alpha-1} \left(\frac{3}{3^{1/\alpha} 4^{1-1/\alpha}+1}\right)^\alpha .
\end{equation*}
which again gives $\Omega\left((6/5)^\alpha \right)$ as lower bound, this time against randomized algorithms.
\end{proof}

We now turn ourselves to the more specific case of the AVR algorithm with the following two results. We recall that the AVR algorithm was shown to be $2^{\alpha-1}\cdot \alpha^\alpha$-competitive by~\citet{DBLP:conf/focs/YaoDS95} in the general deadlines case. In the case of uniform deadlines, the competitive ratio of AVR is actually much better and proofs are much less technical than the original analysis of Yao et al. Recall that for each job $i$ with workload $w_i$, release $r_i$, and deadline $d_i$ ; AVR defines a speed $s_i(t)=w_i/(d_i-r_i)$ if $t\in [r_i,d_i]$ and $0$ otherwise.
\begin{thm}
AVR is $2^\alpha$-competitive for the uniform speed scaling problem.
\label{thm:AVR_ratio}
\end{thm}
\begin{proof}
Let $(w,D,T)$ be a job instance and $s_{\OPT}$ be the speed function of the optimal schedule for this instance. Let $s_{\mathrm{AVR}}$ be the speed function produced by the \textsc{Average Rate} heuristic on the same instance.
It suffices to show that for any time $t$ we have
\begin{equation*}
s_{\mathrm{AVR}}(t)\leq 2 \cdot s_{\OPT}(t) .
\end{equation*}
Fix some $t$.
We assume w.l.o.g.\ that the optimal schedule runs each job $j$ isolated for a total time of $p^*_j$.
By optimality of the schedule, the speed during this time is uniform, i.e.,
exactly $w_j / p^*_j$.
Denote by $j_t$ the job that is processed in the optimal schedule at time $t$. 

Let $j$ be some job with $r_j \le t \le r_j + D$.
It must be that 
\begin{equation}
    \frac{w_{j}}{p^*_j}\leq \frac{w_{j_t}}{p^*_{j_t}} = s_{\OPT}(t) . \label{AVR-monotonicity}
\end{equation}

Note that all jobs $j$ with $r_j \le t \le r_j + D$ are processed completely between
$t - D$ and $t + D$. Therefore,
\begin{equation*}
    \sum_{J_j : r_j \le t \le r_j + D} p^*_j \le 2 D .
\end{equation*}
With (\ref{AVR-monotonicity}) it follows that
\begin{equation*}
    \sum_{J_j : r_j \le t \le r_j + D} w_j \le s_{\OPT}(t) \sum_{J_j : r_j \le t \le r_j + D} p^*_j \le 2 D  \cdot s_{\OPT}(t) .
\end{equation*}
We conclude that
\begin{equation*}
    s_{\mathrm{AVR}}(t) = \sum_{J_j : r_j \le t \le r_j + D} \frac{w_j}{D}
    \le 2\cdot  s_{\OPT}(t) . \qedhere
\end{equation*}
\end{proof}

Next, we show that our upper bound on the exponential dependency in $\alpha$ of the competitive ratio for \textsc{AVR} (in Theorem \ref{thm:AVR_ratio}) is tight for the uniform deadlines case.

\begin{thm}\label{thm:AVR_lowerbound}
Asymptotically ($\alpha$ approaches $\infty$), the competitive ratio of the AVR algorithm for the uniform deadlines case is at least $$\frac{2^\alpha}{e\alpha}$$
\end{thm}
\begin{proof}
Assume $\alpha>2$ and consider a two-job instance with one job arriving at time $0$ of workload $1$ and one job arriving at time $(1-2/\alpha) D$ with workload $1$. One can check that the optimal schedule runs at constant speed throughout the whole instance for a total energy of
\begin{equation*}
\left(\frac{2}{(2-2/\alpha)D}\right)^\alpha \cdot (2-2/\alpha)D .
\end{equation*}
On the other hand, on interval $[(1-2/\alpha)D,D]$, \textsc{AVR} runs at speed $2/D$. This implies the following lower bound on the competitive ratio:
\begin{equation*}
\frac{(2/D)^\alpha \cdot (2/\alpha) D}{\left(\frac{2}{(2-2/\alpha)D}\right)^\alpha \cdot (2-2/\alpha)D}=\frac{2^\alpha}{\alpha}\left(1-\frac{1}{\alpha} \right)^{\alpha-1}
\end{equation*}
which approaches to $2^\alpha / (e\alpha)$ as $\alpha$ tends to infinity.
\end{proof}

\section{Impossibility results for learning augmented speed scaling}
\label{sec:impossibility_results}
This section is devoted to prove some impossibility results about learning augmented algorithms in the context of speed scaling. We first prove that our trade-offs between consistency and robustness are essentially optimal. Again, we describe an instance as a triple $(w,D,T)$.
\begin{thm}\label{thm:robust_consistent_tradeoff}
Assume a deterministic learning enhanced algorithm is $(1+\epsilon/3)^{\alpha-1}$-consistent for any $\alpha\geq 1$ and any small enough constant $\epsilon>0$ (independently of $D$). Then the worst case competitive ratio of this algorithm cannot be better than $\Omega \left(\frac{1}{\epsilon} \right)^{\alpha-1}$.
\end{thm}{}
\begin{proof}
Fix $D$ big enough so that $\lceil \epsilon D \rceil \leq 2\cdot (\epsilon D)$. Consider two different job instances $J_1$ and $J_2$: $J_1$ contains only one job of workload $1$ released at time $0$ and $J_2$ contains an additional job of workload $1/\epsilon$ released at time $\lceil \epsilon D \rceil$. On the first instance, the optimal cost is $1/D^{\alpha-1}$ while the optimum energy cost for $J_2$ is $(1/\lceil \epsilon D \rceil)^{\alpha-1}+D/(\epsilon D)^\alpha\leq (1/\epsilon)^\alpha \cdot  ((1+\epsilon)/D^{\alpha-1})$.

Assume the algorithm is given the job of workload $1$ released at time $0$ and additionally the prediction consists of one job of workload $1/\epsilon$ released at time $\lceil \epsilon D \rceil$. Note that until time $\lceil \epsilon D \rceil$ the algorithm cannot tell the difference between instances $J_1$ and $J_2$. 

Depending on how much the algorithm works before time $\lceil \epsilon D \rceil$, we distinguish the following cases.
\begin{enumerate}
    \item If the algorithm works more that $1/2$ then the energy spent by the algorithm until time $\lceil \epsilon D \rceil$ is at least 
    \begin{equation*}
        (1/2)^\alpha/(\lceil \epsilon D \rceil)^{\alpha-1}=\Omega \left(\frac{1}{\epsilon D}\right)^{\alpha-1}.
    \end{equation*}
    \item However, if it works less than $1/2$ then on instance $J_2$, a total work of at least $(1/\epsilon+1-1/2)=(1/2+1/\epsilon)$ remains to be done in $D$ time units. Hence the energy consumption on instance $J_2$ is at least 
    \begin{equation*}
        \frac{(1/2+1/\epsilon)^\alpha}{D^{\alpha-1}}.
    \end{equation*}
\end{enumerate}
If the algorithm is $(1+\epsilon/3)^{\alpha-1}$-consistent, then it must be that the algorithm works more that $1/2$ before time $\lceil \epsilon D \rceil$ otherwise, by the second case of the analysis, the competitive ratio is at least 
\begin{equation*}
    \frac{(1/2+1/\epsilon)^\alpha}{(1/\epsilon)^\alpha (1+\epsilon)} = \frac{(1+\epsilon/2)^\alpha}{1+\epsilon}> (1+\epsilon/3)^{\alpha-1},
\end{equation*}
where the last inequality holds for $\alpha>4$ and $\epsilon$ small enough.

However it means that if the algorithm was running on instance $J_1$ (i.e. the prediction is incorrect) then by the first case the approximation ratio is at least $\Omega\left(\frac{1}{\epsilon} \right)^{\alpha-1}$.
\end{proof}

We then argue that one cannot hope to rely on some $l_p$ norm for $p<\alpha$ to measure error.
\begin{thm}\label{thm:lower-bound-norm}
Fix some $\alpha$ and $D$ and let $p$ such that $p < \alpha$.
Suppose there is an algorithm which on some prediction $\wPred$
computes a solution of value at most
\begin{equation*}
    C \cdot \OPT + C' \cdot \lVert w - \wPred \rVert_p^{p} .
\end{equation*}
Here $C$ and $C'$ are constants that can be chosen as an arbitrary function of $\alpha$ and $D$.

Then it also exists an algorithm for the online problem (without predictions) which is
$(C +\epsilon)$-competitive for every $\epsilon > 0$.
\end{thm}
In other words, predictions do not help, if we choose $p < \alpha$.
\begin{proof}
In the online algorithm we use the prediction-based algorithm $A_P$ as a black box.
We set the prediction $\tilde w$ to all $0$. We forward each job to $A_P$, but
scale its work by a large factor $M$.
It is obvious that by scaling the optimum of the instance increases exactly by a factor
$M^\alpha$. The error in the prediction, however, increases less:
\begin{equation*}
     \lVert M \cdot w - M \cdot\wPred \rVert_p^{p} = M^{p} \cdot \lVert w - \wPred \rVert_p^{p} .
\end{equation*}
We run the jobs as $A_P$ does, but scale them down by $M$ again. Thus, we get a schedule of value
\begin{equation}
    M^{-\alpha} (M^\alpha \cdot C \cdot \OPT + M^{p} \cdot C' \cdot \lVert w - \wPred \rVert_p^{p}) \\
    = C \cdot \OPT + M^{p - \alpha} \cdot C' \cdot \lVert w - \wPred \rVert_p^{p} . \label{p-lower-alpha}
\end{equation}
Now if we choose $M$ large enough, the second term in (\ref{p-lower-alpha}) becomes insignificant. First, we relate the prediction error to the optimum. First note that
\begin{equation*}
    \OPT \geq (1/D^\alpha)\cdot  ||w||_\alpha^\alpha
\end{equation*}
since the optimum solution cannot be less expensive than running all jobs $i$ disjointly at speed $w_i/D$ for time $D$. Second note that $\lVert w\rVert_p^p \leq ||w||_\alpha^\alpha$
since $|x|^p\le |x|^\alpha$ for any $x\ge 1$ (recall that we assumed our workloads to be integral).
Hence we get that,
\begin{equation*}
    \lVert w - \wPred \rVert_p^{p} = \lVert w \rVert_p^{p} \le D^\alpha \cdot \OPT .
\end{equation*}
Choosing $M$ sufficiently large gives $M^{p-\alpha} C' D^\alpha < \epsilon$, which
implies that (\ref{p-lower-alpha}) is at most $(C+\epsilon)\OPT$.
\end{proof}

\section{Extension to evolving predictors}
\label{sec:evolvingpredictors}
In this section, we extend the result of Section~\ref{sec:algorithm} to the case where the algorithm is provided several predictions over time. In particular, we assume that the algorithm is provided a new prediction at each integral time $t$. The setting is natural as for a very long timeline, it is intuitive that the predictor might renew its prediction over time. Since making a mistake in the prediction of a very far future seems also less hurtful than making a mistake in predicting an immediate future, we define a generalized error metric incorporating this idea. 

Let $0 < \lambda < 1$ be a parameter that describes how fast the confidence in a prediction deteriorates with
the time until the expected arrival of the predicted job. Define the prediction received at time $t$ as a workload vector $\wPred(t)$. Recall we are still considering the uniform deadlines case hence an instance is defined as a triplet $(w,D,T)$.

We then define the total error of a series of predictions  as
\begin{equation*}
    \predErr^{(\lambda)} = \sum_t \sum_{i=t+1}^\infty |\wReal_i - \wPred_i(t)|^\alpha \cdot \lambda^{i - t} .
\end{equation*}

In the following we reduce the evolving
predictions model to the single prediction one.

We would like to prove similar results as in the single prediction setting with respect to $\predErr^{(\lambda)}$.
In order to do so, we will split the instance into parts of bounded time horizon,
solve each one independently with a single prediction, and show that this also gives a guarantee
based on $\predErr^{(\lambda)}$. In particular, we will use the algorithm for the single prediction
model as a black box.

The basic idea is as follows. If no job were to arrive for a duration of $D$, then the instance
before this interval and afterwards can be solved independently. This is because any job in the
earlier instance must finish before any job in the later instance can start. Hence, they cannot interfere.
At random points, we ignore all jobs for a duration of $D$, thereby split the instance.
The ignored jobs will be scheduled sub-optimally using \textsc{AVR}.
If we only do this occasionally, i.e., after every intervals of length $\gg D$,
the error we introduce is negligible.

We proceed by defining the splitting procedure formally.
Consider the timeline as infinite in both directions. To split the instance, we define some interval
length $2 k D$, where $k\in\mathbb N$ will be specified later.
We split the infinite timeline into contiguous intervals of length $2 k D$. Moreover, we choose an offset $x\in\{0,\cdots,k-1\}$ uniformly at random.
Using these values, we define intervals $I_i=[2((i-1) k - x)D, 2(ik - x)D)$.
We will denote by $t_i=(2(i-1)k - x)D$ the start time of the interval $I_i$.
Consequently, the end of $I_i$ is $t_{i+1}$.

In each interval $I_i$, we solve the instance given by the jobs entirely contained in this interval using our algorithm with the most recent prediction as of time $t_i$, i.e.,
$\wPred(t_i)$, and schedule the jobs accordingly. We write $s^{\mathrm{ALG}(i)}$ for this schedule.
For the jobs that are overlapping with two contiguous intervals we schedule them independently
using the \textsc{Average Rate} heuristic.
The schedule for the jobs overlapping with intervals $I_i$ and $I_{i+1}$ will be referred to as $s^{\mathrm{AVR}(i)}$.

It is easy to see that this algorithm is robust: The energy of the produced schedule is
\begin{multline*}
  \int \left(\sum_i \left[ s^{\mathrm{ALG}(i)}(t) + s^{\mathrm{AVR}(i)}(t) \right]\right)^\alpha dt \\
  \le 2^\alpha \int \left(\sum_i s^{\mathrm{ALG}(i)}(t)\right)^\alpha dt 
   +  2^\alpha \int \left(\sum_i s^{\mathrm{AVR}(i)}(t)\right)^\alpha dt .
\end{multline*}
Moreover, the first term can be bounded by $2^\alpha \cdot O(\alpha/\epsilon)^{\alpha} \OPT$
using Theorem~\ref{thm:main_algorithm} and the second term can be bounded by
$2^\alpha \cdot 2^\alpha \OPT$ because of Theorem~\ref{thm:AVR_ratio}.
This gives an overall bound of $O(\alpha/\epsilon)^{\alpha}$ on the competitive ratio.

In the rest of the section we focus on the consistency/smoothness guarantee.
We first bound the costs of $s^{\mathrm{ALG}(i)}$ and $s^{\mathrm{AVR}(i)}$ isolated
(ignoring potential interferences). Using these bounds, we derive an overall guarantee for
the algorithm's cost.

\begin{lem}
\label{lem:cut_avg}
\begin{equation*}
    \mathbb{E} \left( \sum_i \int s^{\mathrm{AVR}(i)}(t)^\alpha\right) \leq \frac{2^{\alpha}}{k} \OPT 
\end{equation*}
\end{lem}
\begin{proof}
Fix some $i>0$ and let us call $O_i$ the job instance consisting of jobs overlapping with both
intervals $I_i$ and $I_{i+1}$. By Theorem~\ref{thm:AVR_ratio} the energy used by \textsc{AVR}
is at most a $2^\alpha$-factor from the optimum schedule. Hence,

\begin{equation*}
    \int s^{\mathrm{AVR}(i)}(t)^\alpha dt \leq 2^\alpha \OPT(O_i) .
\end{equation*}{}

Now denote by $s^{\OPT}$ the speed function of the optimum schedule over the whole instance. Then
\begin{equation*}
    \OPT(O_i) \leq \int_{t_i-D}^{t_i+D} s^{\OPT}(t)^\alpha dt .
\end{equation*}{}
This holds because $s^{\OPT}$ processes some work during $[t_i-D, t_i+D]$ which has to include
all of $O_i$. Hence, we have that
\begin{align*}
    &\mathbb{E}\left(\sum_i\OPT(O_i)\right) \\
    &\le\frac{1}{k} \sum_{x=0}^{k-1} \sum_i \int_{2(ik - x)D - D}^{2(ik - x)D + D} s^{\OPT}(t)^\alpha dt \\
    &\le\frac 1 k \int s^{\OPT}(t)^\alpha dt = \frac 1 k \OPT
\end{align*}
The second inequality holds, because the integrals are over disjoint ranges.
Together, with the bound on $s^{\mathrm{AVR}(i)}$ we get the claimed inequality.
\end{proof}

\begin{lem}\label{lem:split-alg}
\begin{equation*}
 \sum_i \int s^{\mathrm{ALG}(i)}(t)^\alpha dt
 \le (1 + \epsilon) \OPT + O\left(\frac \alpha \epsilon\right)^{\alpha}
     \cdot \lambda^{-2kD} \cdot \predErr^{(\lambda)} .
\end{equation*}
\end{lem}
\begin{proof}
Note that for any $i$
\begin{equation*}
    \sum_{t=(t_i)+1}^{t_{(i+1)}} |\wReal_t-\wPred_t(t_i)|^\alpha
    \le \lambda^{-2kD} \sum_{t=(t_i)+1}^{t_{(i+1)}} |\wReal_t-\wPred_t(t_i)|^\alpha \lambda^{t - t_i} .
\end{equation*}
Hence,
\begin{equation*}
    \sum_i \sum_{t=t_i}^{t_{i+1}} |\wReal_t-\wPred_t(t_i)|^\alpha
    \le \lambda^{-2kD} \predErr^{(\lambda)} .
\end{equation*}
Using Theorem~\ref{thm:main_algorithm} for each
$\int s^{(i)}_{\mathrm{ALG}}(t)^\alpha dt$, we get a bound
depending on $\sum_{t=t_i}^{t_{i+1}} |\wReal_t-\wPred_t(t_i)|^\alpha$.
Summing over $i$ and using the inequality above finishes the proof of the lemma.
\end{proof}
We are ready to state the consistency/smoothness guarantee of the splitting algorithm.
\begin{thm}
\label{thm:splitting_ratio}
With robustness parameter $O(\epsilon / \alpha)$ the splitting algorithm
produces in expectation a schedule of cost at most 
\begin{equation*}
 (1 + \epsilon) \OPT + O\left(\frac \alpha \epsilon \right)^{\alpha}
     \cdot \lambda^{- D/\epsilon \cdot O(\alpha/\epsilon)^\alpha} \cdot \predErr^{(\lambda)} .
\end{equation*}{}
\end{thm}
In other words, we get the same guarantee as in the single prediction case,
except that the dependency on the error is larger by a factor of
$\lambda^{- D/\epsilon \cdot O(\alpha/\epsilon)^\alpha}$. The exponential dependency on $D$
may seem unsatisfying,
but (1) it cannot be avoided (see Theorem~\ref{thm:decay_lower_bound}) and
(2) for moderate values of $\lambda$, e.g. $\lambda = 1 - 1/D$, this exponential dependency vanishes.
\begin{proof}
We will make use of the following inequality: For all $a, b\ge 0$ and $0 < \delta \le 1$, it holds that
\begin{equation*}
    (a + b)^\alpha \le (1 + \delta) a^\alpha + \left(\frac{3\alpha}{\delta}\right)^\alpha b^\alpha .
\end{equation*}
This follows from a simple case distinction whether $b \le a \cdot \delta / (2\alpha)$.
In expectation, the cost of the algorithm is bounded by
\begin{align*}
 \mathbb E\bigg[\int & \left( \sum_i [s^{\mathrm{ALG}(i)}(t) + s^{\mathrm{AVR}(i)}(t)]\right)^\alpha dt \bigg] \\ 
 & \le (1 + \epsilon) \mathbb E\bigg[\int \sum_i (s^{\mathrm{ALG}(i)}(t))^\alpha dt \bigg] \\
  &\quad + \left(\frac{3\alpha}{\epsilon}\right)^{\alpha} \mathbb E\bigg[\int \sum_i (s^{\mathrm{AVR}(i)}(t))^\alpha dt \bigg] \\
 & \le (1 + \epsilon) \int \sum_i s^{\mathrm{ALG}(i)}(t))^\alpha dt \\
  &\quad + \frac 1 k \left(\frac{6\alpha}{\epsilon}\right)^\alpha \OPT .
\end{align*}
By choosing $k = 1/\epsilon (6\alpha / \epsilon)^\alpha$ the latter term becomes $\epsilon \OPT$.
With Lemma~\ref{lem:split-alg} we can bound the term above by
\begin{equation*}
    (1 + \epsilon)^3 \OPT + O\left(\frac{\alpha}{\epsilon}\right)^{\alpha} \cdot \lambda^{-D/\epsilon \cdot O(\alpha/\epsilon)^\alpha}\cdot \predErr^{(\lambda)} .
\end{equation*}
Scaling $\epsilon$ by a constant yields the claimed guarantee.
\end{proof}

We complement the result of this section with an impossibility result. We allow the parameter $\lambda$ in the definition of $\predErr^{(\lambda)}$
to be a function of $D$ and we write $\lambda(D)$.
\begin{thm}\label{thm:decay_lower_bound}
  Let $\predErr{(\lambda)}$ the error in the evolving prediction model be defined with some $0<\lambda(D)<1$ that can depend on $D$.
  Suppose there is an algorithm which computes a solution of value at most
  \begin{equation*}
      C \cdot \OPT + C'(D) \cdot \predErr^{(\lambda)} ,
  \end{equation*}
  where $C$ is independent of $D$ and $C'(D) = o\left(\frac{1-\lambda(D)^D}{\lambda(D)^D} \cdot \frac{1}{D^\alpha} \right)$.
  Then there also exists an algorithm for the online problem (without predictions) which is
$(C +\epsilon)$-competitive for every $\epsilon > 0$.
\end{thm}

In particular, note that for $\lambda$ independent of $D$, it shows that an exponential dependency in $D$ is needed in $C'(D)$ as we get in Theorem~\ref{thm:splitting_ratio}. 
\begin{proof}
The structure of the proof is similar to that of Theorem~\ref{thm:lower-bound-norm}.
We pass an instance to the assumed algorithm, but set the prediction to all $0$.
Unlike the previous proof, we keep the same workloads when passing the jobs,
but subdivide $D$ in to $D \cdot k$ time steps
where $k$ will be specified later.
This will decrease the cost of every
solution by $k^\alpha$.

Take an instance with interval length $D$.
Like in the proof of Theorem~\ref{thm:lower-bound-norm} we have that
\begin{equation*}
    \lVert \wReal \rVert_\alpha^\alpha \le D^\alpha \cdot \OPT .
\end{equation*}
Consider the error parameter $\predErr^{(\lambda)\prime}$ for the instance with $D' = D \cdot k$. We observe that
\begin{align*}
    \predErr^{(\lambda)\prime}& = \sum_t \sum_{i=t+1}^\infty |\wReal_{k\cdot i}|^\alpha \cdot \lambda(D')^{k(i - t)} \\
    & \le ||\wReal||_\alpha^\alpha \cdot \sum_{i=1}^\infty \lambda(D')^{k \cdot i} \\
         & \le  ||\wReal||_\alpha^\alpha  \frac{\lambda(D')^k}{1 - \lambda(D')^k}
         \\
        &
         \le D^\alpha  \frac{\lambda(D')^k}{1 - \lambda(D')^k} \cdot \OPT   
\end{align*}

Hence, by definition the algorithm produces a solution of cost
\begin{equation*}
    C \cdot \OPT / k^\alpha + C'(D') \predErr^{(\lambda)\prime} \le (C / k^\alpha + D^\alpha  \frac{\lambda(D')^k}{1 - \lambda(D')^k} C'(D'))\cdot \OPT
\end{equation*}
for the subdivided instance. Transferring it to the original instance, we get a cost of
\begin{equation*}
    (C + k^\alpha D^\alpha  \frac{\lambda(D')^k}{1 - \lambda(D')^k} C'(D'))\cdot \OPT
\end{equation*}
Therefore, if $k^\alpha \frac{\lambda(D\cdot k)^k}{1 - \lambda(D\cdot k)^k} C'(D\cdot k)$ tends to $0$ as $k$ grows, for any $\epsilon>0$, we can fix $k$ big enough so that the cost of the algorithm is at most $(C+\epsilon)\OPT$.
\end{proof}

\section{A shrinking lemma}
\label{sec:a_shrinking_lemma}
Recall that by applying the earliest-deadline-first policy, we can normalize every schedule to run at most
one job at each time.
We say, it is run \emph{isolated}.
Moreover, if a job is run isolated, it is always better to run it at a uniform speed (by convexity of $x\mapsto x^\alpha$ on $x\ge 0$).
Hence, an optimal schedule can be characterized solely by the total time $p_j$ each job is run.
Given such $p_j$ we will give a necessary and sufficient
condition of when a schedule that runs each job isolated for $p_j$ time exists. Note that we assume we are in the general deadline case, each job $j$ comes with a release $r_j$ and deadline $d_j$ and the EDF policy might cause some jobs to be preempted.
\begin{lem}\label{folklore_lemma}
Let there be a set of $n$ jobs with release times $r_j$ and deadlines $d_j$ for each job $j$.
Let $p_j$ denote the total duration that $j$ should be processed.
Scheduling the jobs isolated earliest-deadline-first, with the constraint to never run a job before its release time, will
complete every job $j$ before time $d_j$ if and only if for every interval $[t, t']$ it holds that
\begin{equation}
    \sum_{j : t \leq r_j,d_j \leq t'} p_j \leq t'-t\label{hall-condition}
\end{equation}
\end{lem}
\begin{proof}
For the one direction, let $t, t'$ such that (\ref{hall-condition}) is not fulfilled.
Since the jobs with $t \leq r_j$ cannot be processed before $t$,
the last such job $j'$ to be completed must finish after
\begin{equation*}
    t + \sum_{j : t \leq r_j,d_j \leq t'} p_j > t+t'-t=t'\geq  d_{j'}
\end{equation*}
For the other direction, we will schedule the jobs earliest-deadline-first and argue
that if the schedule completes some job after its deadline, then (\ref{hall-condition}) is not satisfied for some interval $[t, t']$.

To this end, let $j'$ be the first job that finishes strictly after $d_{j'}$ and consider the interval $I_0=[r_{j'},d_{j'}]$. We now define the following operator that transforms our interval $I_0$ into an interval $I_1$. Consider $t_{\textrm{inf}}$ to be the smallest release time among all jobs that are processed in interval $I_0$ and define $I_1=[t_{\textrm{inf}},d_{j'}]$. We apply iteratively this operation to obtain interval $I_{k+1}$ from interval $I_k$. We claim the following properties that we prove by induction.
\begin{enumerate}
    \item For any $k\geq 0$, the machine is never idle in interval $I_k$.
    \item For any $k\geq 0$, all jobs that are processed in $I_k$ have a deadline $\leq d_{j'}$.
\end{enumerate}
For $I_0=[r_{j'},d_{j'}]$, since job $j'$ is not finished by time $d_{j'}$ it must be that the machine is never idle in that interval. Additionally, if a job is processed in this interval, it must be that its deadline is earlier that $d_{j'}$ since we process in EDF order. Assume both items hold for $I_k$ and then consider $I_{k+1}$ that we denote by $[a_{k+1},d_{j'}]$. By construction, there is a job denoted $j_{k+1}$ released at time $a_{k+1}$ that is not finished by time $a_k$. Therefore the machine cannot be idle at any time in $[a_{k+1},a_k]$ hence at any time in $I_{k+1}$ by the induction hypothesis. Furthermore, consider a job processed in $I_{k+1}\setminus I_k$. It must be that its deadline is earlier that the deadline of job $j_{k+1}$. But job $j_{k+1}$ is processed in interval $I_k$ which implies that its deadline is earlier than $d_{j'}$ and ends the induction. 

Denote by $k'$ the first index such that $I_{k'}=I_{k'+1}$. We define $I_\infty=I_{k'}$. By construction, it must be that all jobs processed in $I_\infty$ have release time in $I_\infty$ and by induction the machine is never idle in this interval and all jobs processed in $I_\infty$ have deadline in $I_\infty$.

Since job $j'$ is not finished by time $d_{j'}$ and by the previous remarks we have that 
\begin{equation*}
    \sum_{j : r_j,d_j\in I_\infty} p_j > |I_\infty|
\end{equation*}
which yields a counter example to \eqref{hall-condition}.
\end{proof}

We can now prove two shrinking lemmas that are needed in the procedure {\sc Robustify} and its generalization to general deadlines. 

\begin{lem}\label{lms:shrink}
Let $0\leq \mu <  1$. For any instance $\mathcal I$ consider the instance $\mathcal I'$ where the deadline of job $j$ is set to $d'_j=r_j+(1-\mu)(d_j-r_j)$ (i.e. we shrink each job by a $(1-\mu)$ factor). Then 
\begin{equation*}
    \OPT(\mathcal I')\leq \frac{\OPT(\mathcal I)}{(1-\mu)^{\alpha-1}}
\end{equation*}
Additionally, assuming $0\leq \mu<1/2$, consider the instance $\mathcal I''$ where the deadline of job $j$ is set to $d''_j=r_j+(1-\mu)(d_j-r_j)$ and the release time is set to $r''_j=r_j+\mu (d_j-r_j)$. Then 
\begin{equation*}
    \OPT(\mathcal I'')\leq \frac{\OPT(\mathcal I)}{(1-2\mu)^{\alpha-1}}
\end{equation*}
\end{lem}
\begin{proof}
W.l.o.g. we can assume that the optimal schedule $s$ for $\mathcal I$ runs each job isolated and
at a uniform speed. By optimality of the schedule and convexity, each job $j$ must be run at a constant speed $s_j$ for a total duration of $p_j$. Consider the first case and define a speed $s'_j=\frac{s_j}{1-\mu}$ for all $j$ (hence the total processing time becomes $p'_j=(1-\mu)\cdot p_j$). 

Assume now in the new instance $\mathcal I'$ we run jobs earliest-deadline-first with the constraint that no job is run before its release time (with the processing times $p'_j$). We will prove using Lemma \ref{folklore_lemma} that all deadlines are satisfied. Consider now an interval $[t,t']$ we then have that
\begin{equation*}
\label{eq:rescheduling}
    \sum_{j : t\leq r_j,d'_j\leq t'} p'_j = (1-\mu) \cdot \sum_{j : t\leq r_j,d'_j\leq t'} p_j  \leq (1-\mu) \cdot \sum_{j : t\leq r_j,d_j\leq \frac{t'-\mu t}{1-\mu}} p_j
\end{equation*} where the last inequality comes from the fact that $t'\geq d'_j = d_j-\mu (d_j-r_j)$ which implies that $d_j\leq \frac{t'-\mu r_j}{1-\mu}\leq  \frac{t'-\mu t}{1-\mu}$ by using $r_j\geq t$.
By Lemma \ref{folklore_lemma} and the fact that $s$ is a feasible schedule for $\mathcal I$ we have that
\begin{equation*}
    \sum_{j : t\leq r_j,d'_j\leq t'} p'_j \leq (1-\mu)\cdot \left(\frac{t'-\mu t}{1-\mu} - t \right) = (1-\mu) \cdot \frac{t'-t}{1-\mu} = t'-t
\end{equation*}
which implies by Lemma \ref{folklore_lemma} that running all jobs EDF with processing time $p'_j$ satisfies all deadlines $d'_j$. Now notice the cost of this schedule is at most $\frac{1}{(1-\mu)^{\alpha-1}}$ times the original schedule $s$ which ends the proof (each job is ran $\frac{1}{1-\mu}$ times faster but for a time $(1-\mu)$ times shorter).

The proof of the second case is similar. Note that for any $[t,t']$, if
\begin{align*}
    d''_j &= r_j+(1-\mu)(d_j-r_j) = (1-\mu)d_j+\mu r_j \leq t'\\
    r''_j &= r_j+\mu(d_j-r_j)=(1-\mu)r_j+\mu d_j \geq t
\end{align*}
then we have 
\begin{align*}
    (1-\mu)d_j &\leq t'-\mu r_j \leq t' - \frac{\mu}{1-\mu}\left(t-\mu d_j \right)\\
    &\iff (1-\mu)d_j -\frac{\mu^2}{1-\mu} d_j \leq t'-\frac{\mu}{1-\mu}\cdot t\\
    &\iff d_j ((1-\mu)^2-\mu^2) \leq (1-\mu) t' - \mu t\\
    &\iff d_j \leq \frac{(1-\mu) t' - \mu t}{1-2\mu}
\end{align*}

Similarly, we have 
\begin{align*}
    (1-\mu)r_j &\geq t-\mu d_j \geq t - \frac{\mu}{1-\mu}\left(t'-\mu r_j \right)\\
    &\iff (1-\mu)r_j-\frac{\mu^2}{1-\mu}r_j \geq t-\frac{\mu}{1-\mu}\cdot t'\\
    &\iff r_j \geq \frac{(1-\mu)t - \mu t'}{1-2\mu}
\end{align*}
Notice that $\frac{(1-\mu) t' - \mu t}{1-2\mu}- \frac{(1-\mu)t - \mu t'}{1-2\mu}=\frac{t'-t}{1-2\mu }$

Therefore, if we set the speed that each job $s''_j$ is processed to $s''_j=\frac{s_j}{1-2\mu}$ then we have a processing time $p''_j=(1-2\mu)\cdot p_j$ and we can write 
\begin{align*}
    \sum_{j : t\leq r''_j,d''_j\leq t'} p''_j &= (1-2\mu)\cdot  \sum_{j : t\leq r''_j,d''_j\leq t'} p_j  \\
    &\leq (1-2\mu)\cdot \sum_{j : \frac{(1-\mu)t-\mu t'}{1-2\mu}\leq r_j,d_j\leq \frac{(1-\mu)t' -\mu t}{1-2\mu}} p_j\\
    &\leq (1-2\mu) \cdot  \frac{t'-t}{1-2\mu} = t'-t
\end{align*}
     by Lemma \ref{folklore_lemma}. Hence we can conclude similarly as in the previous case.
\end{proof}

\newcommand{\wpred}[0]{\ensuremath{w^\mathrm{pred}}\xspace}
\newcommand{\predA}[0]{\ensuremath{\overline{w}^\mathrm{pred}}\xspace}
\newcommand{\wreal}[0]{\ensuremath{w^\mathrm{real}}\xspace}
\newcommand{\wAvail}[0]{\ensuremath{w^\mathrm{avail}}\xspace}
\newcommand{\hwOnline}[0]{\ensuremath{\hat w^\mathrm{online}}\xspace}
\newcommand{\wA}[0]{\ensuremath{\overline w^\mathrm{real}}\xspace}
\newcommand{\errd}[0]{\ensuremath{\mathsf{err}_\eta}\xspace}
\newcommand{\err}[0]{\ensuremath{\mathsf{err}}\xspace}
\newcommand{\NRA}[0]{\textsc{Noise-Robust}\ensuremath{(\mathcal{A})}\xspace}

\section{Making an algorithm noise tolerant}
\label{sec:noise_tolerance}
The idea for achieving noise tolerance
is that by Lemma~\ref{lms:shrink} we know
that if we delay each job's arrival slightly
(e.g., by $\eta\jobLife$) we can
still obtain a near optimal solution.
This gives us time to reassign arriving jobs
within a small interval in order to
make the input more similar to the prediction.
We first, in Section~\ref{sec:noise_tolerant_def}, generalize the error function $\err$ to a more noise tolerant error function $\err_\eta$. 
We then, in Section~\ref{sec:noise_tolerant_alg}, give a general procedure for making an algorithm noise tolerant (see Theorem~\ref{thm:noise-robust}).

\subsection{Noise tolerant measure of error}\label{sec:noise_tolerant_def}

For motivation, recall the example given in the main body. Specifically,  consider a predicted workload $\wPred$  defined by $\wPred_i = 1$ for those time steps $i$ that are divisible by a large constant, say $1000$, and let $\wPred_i =0$ for all other time steps. 
If the real instance $\wReal$ is a small shift  of $\wPred$ say $\wReal_{i+1}= \wPred_i$ then the prediction error $\predErr(\wReal, \wPred)$ is large although $\wPred$ intuitively forms a good prediction of $\wReal$. 
To overcome this sensitivity to noise, we generalize the definition of $\predErr$.

For two workload vectors $w, w'$, and a parameter $\eta \geq 0$, we say that  $w$ is in the $\eta$-neighborhood of $w'$, denoted by $w\in N_\eta(w')$, if $w$ can be obtained from $w'$ by moving the workload at most $\eta D$ time steps forward or backward in time. Formally $w\in N(w')$ if there exists a solution $\{x_{ij}\}$ to the following system of linear equations\footnote{To simplify notation, we assume that $\eta D$ evaluates to an integer and we have extended the vectors $w$ and $w'$ to take value $0$ outside the range $[0, T-D]$.}:
    \begin{align*}
        w_i & = \sum_{j = i - \eta D}^{i+ \eta D} x_{ij}  \qquad \forall i\\
        w'_{j} & = \sum_{i = j - \eta D}^{j+ \eta D} x_{ij}  \qquad \forall j
    \end{align*}
    The concept of $\eta$-neighborhood is inspired by the notion of earth mover's distance but is adapted to our setting. Intuitively, the variable $x_{ij}$ denotes how much of the load $w_i$ has been moved to time unit $j$ in order to obtain $w'$. 
    Also note that it is a symmetric and reflexive relation, i.e., if $w\in N_\eta(w')$ then $w'\in N_\eta(w)$ and $w\in N_\eta(w)$.
    
    We now generalize the measure of prediction error as follows. For a parameter $\eta \geq 0$, an instance $\wReal$, and a prediction $\wPred$, we define the $\eta$-prediction error, denoted by $\predErr_\eta$, as 
\begin{align*}%\label{eq:error measure in the single prediction model}
    \predErr_\eta(\wReal, \wPred) = \min_{w\in N_\eta(\wPred)} \predErr(\wReal, w)\,.
\end{align*}
Note that by symmetry we have that $\predErr_\eta(\wReal, \wPred) = \predErr_\eta(\wPred, \wReal)$. Furthermore, we have that $\predErr_\eta = \predErr$ if $\eta = 0$ but it may be much smaller for $\eta >0$. To see this, consider the vectors $\wPred$ and $\wReal_i = \wPred_{i+1}$ given in the motivational example above. While $\predErr(\wPred, \wReal)$ is large, we have  $\predErr_\eta(\wPred, \wReal) = 0$ for any $\eta$ with $\eta D \geq 1$. Indeed the definition of $\predErr_\eta$ is exactly so as to allow for a certain amount of noise (calibrated by the parameter $\eta$) in the prediction. 

\subsection{Noise tolerant procedure}
\label{sec:noise_tolerant_alg}
We give a general procedure for making an algorithm $\mathcal{A}$ noise tolerant under the mild condition that $\mathcal{A}$ is monotone: we say that an algorithm is monotone if given a predictor $\wPred$ and duration $D$, the cost of scheduling a workload $w$ is at least as large as that of scheduling a workload $w'$ if $w\geq w'$ (coordinate-wise). That increasing the workload should only increase the cost of a schedule is a natural condition that in particular all our algorithms satisfy.
\begin{thm}
    Suppose there is a monotone learning-augmented online algorithm $\mathcal{A}$ for the uniform speed scaling problem, that given prediction $\wPred$, computes a schedule of an instance $\wReal$ of value at most
    \begin{align*}
        \min\{C\cdot \opt + C' \predErr(\wReal, \wPred), C'' \opt\}\,.
    \end{align*}
    Then, for every $\eta \ge 0$, $\zeta>0$
    there is a learning-augmented online algorithm $\NRA$, that given prediction $\wPred$, computes a schedule of $\wReal$ of value at most  $((1 + \eta)(1 + \zeta))^{O(\alpha)}$ times
    \begin{align*}
        \min\{C\cdot \opt + (1/\zeta)^{O(\alpha)}(C + C') \predErr_\eta(\wReal, \wPred), C''\opt\}\,.
    \end{align*}
    \label{thm:noise-robust}
\end{thm}

The pseudo-code of the online algorithm \NRA, obtained from $\mathcal{A}$, is given in Algorithm~\ref{alg:NRA}. 
\renewcommand{\algorithmicwhile}{\textbf{on time step}}
\renewcommand{\algorithmicendwhile}{\algorithmicend\ \textbf{on}}
\begin{algorithm}[tb]
   \caption{\NRA}
   \label{alg:NRA}
\begin{algorithmic}[1]
   \REQUIRE Algorithm $\mathcal{A}$, prediction $\wPred$, and $\eta \geq 0, \zeta>0$
   \STATE Initialize $\mathcal{A}$ with prediction $\predA_i = (1+\zeta)\wPred_{i - \eta\jobLife}$ and duration $(1-2\eta)\jobLife$
   \STATE Let $\wOnline$ and $\wA$ be workload vectors, initialized to $0$
   \WHILE{$i$}
        \STATE $W \gets \wReal_i$
        \FOR{$j \in \{i-\eta\jobLife, \dotsc, i + \eta \jobLife\}$}
        \IF{$\wOnline_j + W \le (1+\zeta)\wPred_j$}
        \STATE $x_{ij} \gets W$
        \STATE $W \leftarrow 0$
        \STATE $\wOnline_j \gets \wOnline_j + W$
        \ELSIF{$\wOnline_j < (1+\zeta)\wPred_j$}
                \STATE $x_{ij} \gets (1+\zeta)\wPred_j - \wOnline_j$
        \STATE $W \leftarrow W - x_{ij}$
        \STATE $\wOnline_j \gets (1+\zeta)\wPred_j$
        \ENDIF
        \ENDFOR
        \STATE \textit{// Distribute remaining workload $W$ evenly}
        \FOR{$j \in \{i-\eta\jobLife, \dotsc, i + \eta \jobLife\}$}
            \STATE $x_{ij} \leftarrow x_{ij} + W/(2\eta \jobLife + 1)$
            \STATE $\wOnline_j \leftarrow \wOnline_j + W/(2\eta \jobLife + 1)$
        \ENDFOR
        \STATE $\wA_i \gets \wOnline_{i-\eta\jobLife}$
        \STATE Feed the job with workload $\wA_i$ to $\mathcal{A}$
   \ENDWHILE
\end{algorithmic}
\end{algorithm}

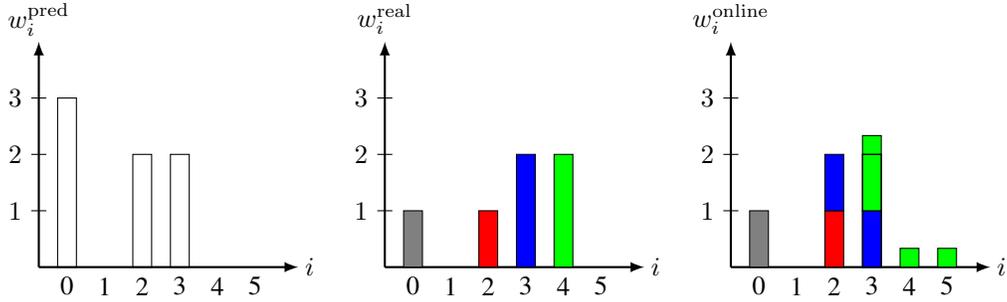
\begin{figure}[bt]
    \centering
    \tikzset{>=latex}
    \begin{tikzpicture}
    
   %%%%% PREDICTION %%%%%% 
        %%%% AXIS
        \draw (0,0) edge[thick, ->] (3.45,0);
        \node at (3.6, 0) {$i$};
        \draw (0,0) edge[thick, ->] (0,3);
        \node at (0, 3.3) {$\wPred_i$};
        \draw (-0.1, 0.75) edge node[left=0.1cm]{$1$} (0.1, 0.75);
        \draw (-0.1, 1.5) edge node[left=0.1cm]{$2$}(0.1, 1.5);
        \draw (-0.1, 2.25) edge node[left=0.1cm]{$3$}(0.1, 2.25);
    
        %%% JOBS    
        \draw (0.25, 0) rectangle (0.5, 2.25);
        \node at (0.375, -0.25) {0};
        
%        \draw (0.75, 0) rectangle (1.0, 1.5);
        \node at (0.875, -0.25) {1};
        
        \draw (1.25, 0) rectangle (1.5, 1.5);
        \node at (1.375, -0.25) {2};
        
        \draw (1.75, 0) rectangle (2.0, 1.5);
        \node at (1.875, -0.25) {3};

%        \draw (2.25, 0) rectangle (2.5, 1.5);
        \node at (2.375, -0.25) {4};
        
 %       \draw (2.75, 0) rectangle (3.0, 1.5);
        \node at (2.875, -0.25) {5};
        
   %%%%% REAL INSTANCE %%%%%% 
   \begin{scope}[xshift=4.6cm]
        %%%% AXIS
        \draw (0,0) edge[thick, ->] (3.45,0);
        \node at (3.6, 0) {$i$};
        \draw (0,0) edge[thick, ->] (0,3);
        \node at (0, 3.3) {$\wReal_i$};
        \draw (-0.1, 0.75) edge node[left=0.1cm]{$1$} (0.1, 0.75);
        \draw (-0.1, 1.5) edge node[left=0.1cm]{$2$}(0.1, 1.5);
        \draw (-0.1, 2.25) edge node[left=0.1cm]{$3$}(0.1, 2.25);
    
        %%% JOBS    
        \draw[fill=gray] (0.25, 0) rectangle (0.5, 0.75);
        \node at (0.375, -0.25) {0};
        
%        \draw (0.75, 0) rectangle (1.0, 1.5);
        \node at (0.875, -0.25) {1};
        
        \draw[fill=red] (1.25, 0) rectangle (1.5, 0.75);
        \node at (1.375, -0.25) {2};
        
        \draw[fill=blue] (1.75, 0) rectangle (2.0, 1.5);
        \node at (1.875, -0.25) {3};

        \draw[fill=green] (2.25, 0) rectangle (2.5, 1.5);
        \node at (2.375, -0.25) {4};
        
 %       \draw (2.75, 0) rectangle (3.0, 1.5);
        \node at (2.875, -0.25) {5};
        \end{scope}
   
   %%%%% CONSTRUCTED W %%%%%%
        \begin{scope}[xshift=9.2cm]
        %%%% AXIS
        \draw (0,0) edge[thick, ->] (3.45,0);
        \node at (3.6, 0) {$i$};
        \draw (0,0) edge[thick, ->] (0,3);
        \node at (0, 3.3) {$\wOnline_i$};
        \draw (-0.1, 0.75) edge node[left=0.1cm]{$1$} (0.1, 0.75);
        \draw (-0.1, 1.5) edge node[left=0.1cm]{$2$}(0.1, 1.5);
        \draw (-0.1, 2.25) edge node[left=0.1cm]{$3$}(0.1, 2.25);
    
        %%% JOBS    
        
        \draw[fill=gray] (0.25, 0) rectangle (0.5, 0.75);
%        \draw[dashed, ultra thick] (0.25, 0) rectangle (0.5, 2.25);
        \node at (0.375, -0.25) {0};
        
%        \draw (0.75, 0) rectangle (1.0, 1.5);
        \node at (0.875, -0.25) {1};
        
        \draw[fill=red] (1.25, 0) rectangle (1.5, 0.75);
        \draw[fill=blue] (1.25, 0.75) rectangle (1.5, 1.5);
 %       \draw[dashed, ultra thick] (1.25, 0) rectangle (1.5, 1.5);
        \node at (1.375, -0.25) {2};
        
        \draw[fill=blue] (1.75, 0) rectangle (2.0, 0.75);
        \draw[fill=green] (1.75, 0.75) rectangle (2.0, 1.5);
        \draw[fill=green] (1.75, 1.5) rectangle (2.0, 1.75);
        \draw (1.75, 0) rectangle (2.0, 1.5);
        \node at (1.875, -0.25) {3};

        \draw[fill=green] (2.25, 0) rectangle (2.5, 0.25);
        \node at (2.375, -0.25) {4};
        
 %       \draw (2.75, 0) rectangle (3.0, 1.5);
        \draw[fill=green] (2.75, 0) rectangle (3.0, 0.25);
        \node at (2.875, -0.25) {5};
        \end{scope}
    \end{tikzpicture}
    \caption{An example of the construction of the vector $\wOnline$ from $\wReal$ and $\wPred$.}
    \label{fig:NRAexample}
\end{figure}
 The algorithm constructs a vector $\wOnline\in N_\eta(\wReal)$
 while trying to minimize 
 $\err(\wOnline, \wpred)$.
 Each component $\wOnline_i$ will be finalized at time $i + \eta\jobLife$.
 Hence, we forward the jobs to $\mathcal A$ with a delay of $\eta\jobLife$.
 
 The vector is constructed as follows. 
 Suppose a job $\wreal_i$ arrives.
 The algorithm first (see Steps~$4$-$15$) greedily assigns the workload to the time steps $j = i - \eta D, i- \eta D + 1, \ldots, i+ \eta D$ from left-to-right subject to the constraint that no time step receives a workload higher than $(1+\zeta) \wPred_j$. 
 If not all workload of $\wReal_i$ was assigned in this way, then the overflow is assigned uniformly to the time steps from $i-\eta D$ to $i+ \eta D$ (Steps~$17$-$20$).
 Since each $\wOnline_j$ can only
 receive workloads during time steps $j-\eta\jobLife,\dotsc,j+\eta\jobLife$,
 it will be finalized at time $j + \eta\jobLife$.
 Thus, at time $i$ we can safely forward $\wOnline_{i-\eta\jobLife}$ to the algorithm $\mathcal A$. Hence, we
 set the workload of the algorithm's instance to
 $\wA_i = \wOnline_{i-\eta\jobLife}$ (Steps~$21$-$22$).
 This shift together with the fact that a job $\wReal_i$ may be assigned to $\wOnline_{i+\eta \jobLife}$, i.e., $\eta \jobLife$ time steps forward in time, is the reason why we run each job with an interval of length $(1-2\eta)D$. Shrinking the interval of each job allows to make this shift and reassignment while still guaranteeing that each job is finished by its original deadline. 
 
 For an example, consider Figure~\ref{fig:NRAexample}. Here we assume that $\eta D = 1$ and for illustrative purposes that $\zeta = 0$. At time $0$, a workload $\wReal_0 = 1$ is released.
 The algorithm $\NRA$ then greedily constructs $\wOnline$ by filling the available slots in $\wPred_{-1}, \wPred_0,$ and $\wPred_1$. Since $\wPred_0 = 3$, it fits all of the workload of $\wReal_0$  at time $0$. Similarly the workloads $\wReal_2$ and $\wReal_3$ both fit under the capacity given by $\wPred$. Now consider the workload $\wReal_4 = 2$ released at time $4$. At this point, the available workload at time $2$ is
 fully occupied and one there is one unit of
 workload left at time $3$.
 Hence, \NRA will first assign the one unit of $\wReal_4$ to the third time slot and then split the remaining unit of workload unit uniformly across the time steps $3,4,5$. The obtained vector $\wOnline$ is depicted on the right of Figure~\ref{fig:NRAexample}. The workload $\wOnline$ is then fed online to the algorithm $\mathcal{A}$ (giving a schedule of $\wOnline$ and thus of $\wReal$) so that at time $i$, $\mathcal{A}$ receives the job $\wA_i = \wOnline_{i+\eta D} = \wOnline_{i+1}$ with a deadline of $i + (1-2\eta) D = i + D -2$. This deadline is chosen so as to guarantee that a job is finished by $\mathcal{A}$ within its original deadline. Indeed, by this selection, the last part of the job $\wReal_4$ that was assigned to $\wOnline_5$ is guaranteed to finish by time $6 + D -2 = 4 + D$ which is its original deadline.
 
 Having described the algorithm, we proceed to analyze its guarantees which will prove Theorem~\ref{thm:noise-robust}.
 
 \paragraph{Analysis.} 
 We start by analyzing the noise tolerance of \NRA.
\begin{lem}
    The schedule computed by $\NRA$ has cost at most $(1+O(\eta))^\alpha C''\OPT$.
\end{lem}
\begin{proof}
    Let $\opt$ and $\opt'$ denote the cost of an optimum schedule of the original instance $\wReal$ with duration $D$ and the instance $\wA$ with duration $(1-2\eta)D$ fed to $\mathcal{A}$, respectively. The lemma then follows by showing that 
    \begin{align*}
        \opt' \leq (1+O(\eta))^{\alpha} \opt\,.
    \end{align*}
    To show this inequality, consider an optimal schedule $s$ of $\wReal$ subject to the constraint that every job $\wReal_i$ is scheduled within the time interval $[i + 2\eta \jobLife, i + (1-2\eta)\jobLife]$. By Lemma \ref{lms:shrink}, we have that the cost of this schedule is at most $(1+O(\eta))^{\alpha}\opt$. 
    The statement therefore follows by arguing that $s$ also gives a feasible schedule of $\wA$  with duration $(1-2\eta)D$. To see this note that $\NRA$ moves the  workload $\wReal_i$ to a subset of $\wA_{i}, \wA_{i+1}, \ldots, \wA_{i + 2\eta D}$. All of these jobs are allowed to be processed during $[i+2\eta D, i + (1-2\eta)D]$.
    It follows that the part of these jobs that corresponds to $\wReal_i$ can be processed in the computed schedule $s$ (whenever it processes $\wReal_i$) since $s$ process that job in the time interval $[i + 2\eta D, i+(1-2\eta)D]$. By doing this ``reverse-mapping'' for every job, we can thus use $s$ as a schedule for the instance $\wA$ with duration $(1-2\eta)D$. 
\end{proof}

We now proceed to analyze the consistency and smoothness. The following lemma is the main technical part of the analysis.  We use the common notation $(a)^+$ for  $\max\{a, 0\}$.
\begin{lem}
    The workload vector $\wOnline$ produced by \NRA satisfies
    \begin{align*}
        \sum_i \left[\left(\wOnline_i - (1+\zeta)\wPred_i\right)^+\right]^\alpha \leq  O(1/\zeta)^{3\alpha} \cdot \min_{w\in N_\eta(\wReal)} \sum_i \left[\left(w_i - \wPred_i\right)^+\right]^\alpha\,.
    \end{align*}
    \label{lem:workload_robust}
\end{lem}
The more technical proof of this lemma is given in Section~\ref{sec:robust_prooflemma}. Here, we explain how it implies the  consistency and smoothness bounds of Theorem~\ref{thm:noise-robust}. For a workload vector $w$, we use the notation $\OPT(w)$ and $\OPT'(w)$ to denote the cost of an optimal schedule of workload $w$ with duration $D$ and $(1-2\eta)D$, respectively. Now let $\hwOnline$ be the workload vector defined by
\begin{align*}
    \hwOnline_i = \max \{ \wOnline_i, (1+\zeta) \wPred_i\}\,. 
\end{align*}
We analyze the cost of the schedule produced
by $\mathcal A$ for $\hwOnline$ (shifted by $\eta\jobLife$).
This also bounds the cost of running
 $\mathcal A$ with $\wA$:
Since $\mathcal{A}$ is monotone, the cost of the schedule computed for the workload $\hwOnline$ (shifted by $\eta D$) can only be greater than that computed for $\wA$ which equals $\wOnline$ (shifted by $\eta D$).
Furthermore, we have by Lemma~\ref{lem:workload_robust} that 
\begin{align}
    \err(\hwOnline, (1+\zeta) \wPred)
    &=  \sum_i \left[\left(\wOnline_i - (1+\zeta)\wPred_i\right)^+\right]^\alpha \label{eq:robust_error} \\
    &\leq O(1/\zeta)^{3\alpha}  \err_\eta (\wReal, \wPred)\,.
    \notag
\end{align}
It follows by the assumptions on $\mathcal{A}$ that the schedule computed by \NRA has cost at most
\begin{align*}
   C \cdot \OPT'(\hwOnline) & + C' \cdot \err(\hwOnline, (1+\zeta) \wPred) \\
   & \leq C \cdot \OPT'(\hwOnline) + O(1/\zeta)^{3\alpha} \cdot C' \cdot \err_\eta(\wreal, \wPred)\,.
\end{align*}

The following lemma implies the consistency and smoothness, as stated in Theorem~\ref{thm:noise-robust}, by relating $\OPT'(\hwOnline)$ with the cost $\OPT = \OPT(\wReal)$.

% \label{smooth-schedule-1}
\begin{lem}
    We have 
    \begin{align*}
        \OPT'(\hwOnline) \leq ((1+\eta)(1+\zeta))^{O(\alpha)} \left(  \OPT(\wReal) 
          + O(1/\zeta)^{4\alpha}\predErr_\eta(\wReal, \wPred)\right)\,.
    \end{align*}
\end{lem}
\begin{proof}
    By the exact same arguments as in the proof of Theorem~\ref{thm:cost-trust}, we have that for any   $\eta'>0$
    \begin{align*}
        \OPT'(\hwOnline) &\leq  (1+\eta')^\alpha \OPT'((1+\zeta)\wPred) 
          + O(1/\eta')^\alpha \predErr(\hwOnline, (1+\zeta)\wPred)\\
        & \leq  (1+\eta')^\alpha \OPT'((1+\zeta)\wPred) 
          + O(1/\eta')^\alpha O(1/\zeta)^{3\alpha} \predErr_\eta(\wReal, \wPred)\,,
    \end{align*}
    where we used~\eqref{eq:robust_error} for the second inequality.
    
    By Lemma~\ref{lms:shrink}, we have that decreasing the duration by a factor $(1-2\eta)$ only increases the cost by factor $(1+O(\eta))^{\alpha}$ and so  $\OPT'((1+\zeta)\wPred) \leq (1+O(\eta))^\alpha \OPT((1+\zeta)\wPred)$. Furthermore, as a schedule for a workload $\wPred$ gives a schedule for $(1+\zeta) \wPred$ by increasing the speed by a factor $(1+\zeta)$, we get
    \begin{align*}
        \OPT'((1+\zeta)\wPred) \leq  (1+O(\eta))^\alpha (1+\zeta)^\alpha \OPT(\wPred)\,.
    \end{align*}
    Hence, by choosing $\eta' = \zeta$, 
    \begin{align*}
        \OPT'(\hwOnline) \leq  (1+O(\eta))^\alpha (1+\zeta)^{2\alpha} \OPT(\wPred) 
          + O(1/\zeta)^{4\alpha}\predErr_\eta(\wReal, \wPred) .
    \end{align*}
    
    It remains to upper bound $\OPT(\wPred)$ by $\OPT(\wReal)$.
    Let $w = \argmin_{w\in N_\eta(\wPred)}\predErr(w,\wReal)$  and so $\predErr_\eta(\wReal, \wPred) = \predErr(\wReal, w)$. By again applying the arguments of Theorem~\ref{thm:cost-trust}, we have for any $\eta'>0$
    \begin{align*}
        \OPT(w) \leq (1+\eta')^\alpha \OPT(\wReal) + O(1/\eta')^{\alpha} \predErr(\wreal,w)\,.
    \end{align*}
Now consider an optimal schedule of $w$ subject to  that for every time $t$ the job $w_t$ is scheduled within the interval $[t+ \eta \jobLife, t + (1-\eta)\jobLife]$. 
By Lemma~\ref{lms:shrink}, we have that this schedule has cost at most $(1+O(\eta))^\alpha \OPT(w)$.
Observe that this schedule for $w$ also defines a feasible schedule for $\wPred$ since the time of any job is shifted by at most $\eta D$ in $w$. Hence, by again selecting $\eta' = \zeta$,
\begin{align*}
    \opt(\wPred)& \leq (1+O(\eta))^\alpha \OPT(w) \\
    &\leq (1+ O(\eta))^\alpha \left( (1+\zeta)^\alpha \OPT(\wReal) + O(1/\zeta)^{\alpha} \predErr_\eta(\wReal,\wPred)\right)
\end{align*}
    Finally,  by combining all inequalities, we get
    \begin{align*}
        \OPT'(\hwOnline) \leq (1+O(\eta))^{2\alpha} \left( (1+\zeta)^{3\alpha} \OPT(\wReal) 
          + O(1/\zeta)^{4\alpha}\predErr_\eta(\wReal, \wPred)\right)
    \end{align*}
\end{proof}

\subsubsection{Proof of Lemma~\ref{lem:workload_robust}}
\label{sec:robust_prooflemma}
The lemma is trivially true if there were no jobs that had remaining workloads to be assigned uniformly, i.e., if we always have $W = 0$ at Step~$16$ of \NRA.
So suppose that there was at least one such job and 
  consider the directed bipartite graph $G$ with bipartitions $A$ and $B$ defined as follows:
\begin{itemize}
    \item $A$ contains a vertex for each component of $\wreal$ and $B$ contains one for each component
    of $\wOnline$. In other words, $A$ and $B$
    contain one vertex for each time unit.
    \item There is an arc from $i\in A$ to $j\in B$
    if $|i - j| \le \eta \jobLife$, that is,
    if $\wreal_i$ could potentially be assigned to $\wOnline_j$.
    \item There is an arc from $j\in B$ to $i\in A$ if part of the workload of $\wReal_i$ was assigned to $\wOnline_j$ by $\NRA$, i.e., if $x_{ij} >0$.
\end{itemize}

Now let $t$  be the \emph{last} time step such that the online algorithm had to assign the remaining workload of $\wReal_{t}$ uniformly. 
So, by selection, $t + \eta D$ is the last time step so that $\wOnline_{t + \eta D}> (1+\zeta)  \wPred_{t + \eta D}$. For $k\geq 0$, define the sets
\begin{align*}
    A_k & = \{i\in A: \mbox{the shortest path from $t$ to $i$ has length $2k$ in $G$}\}, \\
    B_k & = \{j \in B: \mbox{the shortest path from $t$ to $j$ has length $2k+1$ in $G$}\}. 
\end{align*}
Here $t$ stands for the corresponding vertex in $A$.
The set $A_k$ consists of those time steps, for which the corresponding jobs in $\wReal$ have been moved in $\wOnline$ to the time slots in $B_{k-1}$ but not to any time slot in $B_{k-2}, B_{k-3}, \ldots, B_{0}$; and $B_k$ are all the time slots where the jobs corresponding to $A_{k}$ could have been assigned (but no job in $A_{k-1}, A_{k-2}, \ldots, A_0$ could have been assigned). By the selection of $t$, and the construction of $\wOnline$, these sets satisfy the following two properties:
\begin{claim} The sets $(A_k, B_k)_{k\geq 0}$ satisfy
    \begin{itemize}
        \item For any time step $j\in \bigcup_k B_k$ we have $\wOnline_j \geq (1+\zeta) \wPred_j$.
        \item For any two time steps $i_k \in A_k$ and $i_\ell\in A_\ell$ with $k>\ell$, we have $i_k - i_\ell \leq 2\eta D (k-\ell+2)$.
    \end{itemize}
\end{claim}
\begin{proof}[Proof of claim.]
    In the proof of the claim we use the notation $\ell(A_k)$ and $\ell(B_k)$ to denote the left-most (earliest) time step in $A_k$ and $B_k$, respectively.
    The proof is by induction on $k\geq 0$ with the following induction hypothesis (IH):
    \begin{enumerate}
        \item For any time step $j\in B_k$ we have $\wOnline_j \geq (1+\zeta) \wPred_j$.
        \item $B_0 = \{t-\eta D, \ldots, t+ \eta D\}$ and for any (non-empty) $B_k$ with $k> 1$ we have $B_k = \{\ell(B_k), \ldots, \ell(B_{k-1})-1\}$ and $\ell(B_k) - \ell(B_{k-1}) \leq 2\eta D$. 
    \end{enumerate}
    The first part of IH immediately implies the first part of the claim. The second part implies the second part of the claim as follows: Any time step in $A_{\ell}$ has a time step in $B_\ell$ that differs by at most $\eta D$. Similarly, for any time step in $A_k$ there is a time step in $B_{k-1}$ at distance at most $\eta D$. Now by the second part of the induction hypothesis, the distance between these time steps in $B_{k-1}$ and $B_\ell$ is at most $(k-\ell+1)2\eta D$.  
    
    We complete the proof by verifying the inductive hypothesis. For the base case when $k=0$, we have $B_0 = \{t-\eta D, \ldots, t+\eta D\}$ by definition since $A_0 = \{t\}$. We also have that the first part of IH holds by the definition of \NRA and the fact that the overflow of job $\wReal(t)$ was uniformly assigned to these time steps.  
    
    For the inductive step, consider a time step $i\in A_{k}$. By definition $\wReal_i$ was  assigned to a time step in $B_{k-1}$ but to no time step in $B_{k-2} \cup \ldots \cup B_0$. Now suppose toward contradiction that there is a time step $j\in A_{k-1}$ such that $j< i$. But then by the greedy strategy of $\NRA$ (jobs are assigned left-to-right), we reach the contradiction that $\wReal_i$ must have been assigned to a time step in $B_{k-2} \cup \ldots \cup B_0$ if $k\geq 2$ since then $\wReal_j$ is assigned to a time step in $B_{k-2}$. For $k=1$, we have $j=t$ and so all time steps in $B_0$ were full (with respect to capacity $(1+\zeta) \wPred$) after $t$ was processed. Hence, in this case, $\wReal_i$ could only be assigned to a time step in $B_0$ if it it had overflow that was uniformly assigned by \NRA, which contradicts the selection of $t$. 
    
    We thus have that each time step in $A_k$ is smaller than the earliest time step in $A_{k-1}$. It follows that $B_k = \{\ell(B_k), \ldots, \ell(B_{k-1}) -1\}$ where $\ell(B_k) = \ell(A_k) - \eta D$. The bound $\ell(B_k)- \ell(B_{k-1}) \leq 2\eta D$ then follows since, by definition, $\{\ell(A_k) - \eta D, \ldots, \ell(A_k) + \eta D\}$ must intersect $B_{k-1}$. This completes the inductive step for the second part of IH. For the first part, note that the job  $\wReal_{\ell(A_k)}$  was also assigned to $B_{k-1}$ by \NRA. By the greedy left-to-right strategy, this only happens if the capacity of all time steps $B_k$ is saturated. 
\end{proof}

Now let $p$ be the smallest index such that $\wReal(A_{p+1}) + \wReal(A_{p+2}) \leq \zeta' \sum_{i=0}^p \wReal(A_i)$ where we select $\zeta'=\zeta/10$. We have
\begin{align}
    \sum_{i=0}^{p+1}\wReal(A_{i}) \geq \sum_{i=0}^p \wOnline(B_i) \geq  (1+\zeta) \sum_{i=0}^p \wPred(B_i)
    \label{eq:perturbations_1}
\end{align}
where the first inequality holds by the definition of the sets and the second is by the first part of the above claim.
 In addition, by the selection of $p$,
\begin{align}
    \sum_{i=0}^p\wReal(A_{i}) \geq  (1-\zeta') \sum_{i=0}^{p+2}\wreal(A_{i})\,.
    \label{eq:perturbatoins_2}
\end{align}
Now let $q = \max\{p - 4/(\zeta')^2, 0\}$. We claim the following inequality
\begin{align}
    \sum_{i=q}^p\wReal(A_{i}) \geq  (1-\zeta') \sum_{i=0}^p \wReal(A_{i})\,.
    \label{eq:exp_decay}
\end{align}
 The inequality is trivially true if $q = 0$. Otherwise, we have by the selection of $p$,
\begin{align*}
    \sum_{i=q}^p\wReal(A_i) &  = (1-\zeta')\sum_{i=q}^p \wReal(A_i) + \zeta'\sum_{i=q}^p  \wReal(A_i) \\
    & \geq (1-\zeta') \sum_{i=q}^p \wReal(A_i) + \frac{(p-q)}{2} (\zeta')^2\sum_{i=0}^{q-1}\wReal(A_i) \\
    & \geq (1-\zeta')\sum_{i=q}^p \wReal(A_i) + 2 \sum_{i=0}^{q-1}\wReal(A_i)
\end{align*}
and so~\eqref{eq:exp_decay} holds.

We are now ready to complete the proof of the lemma.  Let $w^*$ be a minimizer of the right-hand-side, i.e., 
\begin{align*}
    w^* = \argmin_{w\in N_\eta(\wReal)} \sum_i \left[\left(w_i - \wPred_i\right)^+\right]^\alpha
\end{align*}
Divide the time steps of the instance into $T_1, B_{p+1}, T_2$ and $T_3$ where $T_1$ contains all time steps earlier than $\ell(B_{p+1})$, $T_2$ contains the time steps in $\cup_{i=0}^{p} B_i$, and $T_3$ contains the remaining time steps, i.e., those after $t+\eta D$.
By the selection of $t$, we have  $\wOnline_i \leq (1+\zeta) \wPred_i$ for all $i\in T_3$. We thus have that $\sum_i \left[\left(\wOnline_i - (1+\zeta)\wpred_i\right)^+\right]^\alpha$ equals
\begin{align*}
    \sum_{i \in T_1} \left[\left(\wOnline_i - (1+\zeta)\wpred_i\right)^+\right]^\alpha + 
    \sum_{i \in B_{p+1} \cup T_2} \left[\left(\wOnline_i - (1+\zeta)\wpred_i\right)^+\right]^\alpha\,. 
\end{align*}
We start by analyzing the second sum.  The only jobs in $\wReal$ that contribute to the workload of $\wOnline$ at the time steps in $B_{p+1} \cup T_2$  are by definition those corresponding to time steps in $A_0 \cup \ldots \cup A_{p+2}$. In the worst case, we have that $\wPred$ is $0$ during these time steps and that the jobs in $\wReal$ are uniformly assigned to the same $2\eta D + 1$ time steps. This gives us the upper bound:
\begin{align*}
    \sum_{i \in B_{p+1} \cup T_2} \left[\left(\wOnline_i - (1+\zeta)\wpred_i\right)^+\right]^\alpha &\leq \left(\frac{\sum_{i=0}^{p+2}\wReal(A_i)}{2\eta D+1} \right)^\alpha\cdot (2\eta D + 1) \\
    &\leq (1+\zeta')^\alpha\left(\frac{\sum_{i=0}^{p}\wReal(A_i)}{2\eta D} \right)^\alpha 2\eta D\,.
\end{align*}
At the same time, combining~\eqref{eq:perturbations_1}~\eqref{eq:perturbatoins_2}, and~\eqref{eq:exp_decay} give us
\begin{align*}
    \sum_{i=q}^p \wReal(A_i)\geq (1-\zeta')^2 (1+\zeta) \sum_{i=0}^p \wPred(B_i) \geq (1+\zeta/2) \sum_{i=0}^p \wpred(B_i)\,.
\end{align*}
By definition, the jobs in $\wReal$ corresponding to time steps $\cup_{k=q}^p A_k$ can only be assigned to $\wOnline$ during time steps $T_2 = \cup_{k=0}^p B_k$. Therefore, as  the difference between the largest time and smallest time in $\cup_{k=q}^p A_k$ is at most $2\eta D (p-q+2)$ (second statement of the above claim) and thus the workload of those time steps can be assigned to at most $2\eta D(p-q+4)$ time steps, we have  
\begin{align*}
     \sum_{i\in T_2} \left[\left(w^*_i - \wPred_i\right)^+\right]^\alpha& \geq  \left(\frac{\sum_{i=q}^p \wReal (A_i) - \sum_{i=0}^p \wpred(B_i)}{(p-q+4)\cdot 2\eta D}\right)^\alpha \cdot (p-q+4) \cdot 2\eta D  \\
   &\geq  \left(c\cdot \zeta^{3}\right)^{\alpha}  \left(\frac{\sum_{i=0}^p\wReal(A_{i})}{ 2\eta D} \right)^\alpha \cdot 2\eta D 
\end{align*}
for an absolute constant $c$. It follows that 
\begin{align*}
    \sum_{i \in B_{p+1} \cup T_2} \left[\left(\wOnline_i - (1+\zeta)\wpred_i\right)^+\right]^\alpha \leq \left(\frac{1+\zeta'}{c\zeta^3} \right)^\alpha \sum_{i\in T_2} \left[\left(w^*_i - \wPred_i\right)^+\right]^\alpha\,.
\end{align*}
We have thus upper bounded the sum on the left over time steps in $B_{p+1} \cup T_2$ by the sum on the right over only time steps in $T_2$. Since $\NRA$ does not assign the workload $\wReal_i$ for $i\in T_1$ to $\wOnline$ on any of the time steps in $T_2$, we  can repeatedly apply  the arguments on the time steps in $T_1$ to show
\begin{align*}
    \sum_{i \in T_1} \left[\left(\wOnline_i - (1+\zeta)\wpred_i\right)^+\right]^\alpha \leq \left(\frac{1+\zeta'}{c\zeta^3} \right)^\alpha \sum_{i\in T_1 \cup B_{p+1}} \left[\left(w^*_i - \wPred_i\right)^+\right]^\alpha\,,
\end{align*}
yielding the statement of the lemma.

\section{\robustalg for uniform deadlines}
\label{sec:appendix_convolution}
Here we provide the proofs of Claim~\ref{cla:convolution1}, Claim~\ref{cla:convolution_energy}, Claim~\ref{cla:convolution_robust}.

\CONVcla*
\begin{proof}
Since $s$ is a feasible schedule for $(w,(1-\delta D),T)$, we have that
\begin{equation*}
    \int_{r_i}^{r_i+D} s_i^{(\delta)}(t) dt = \int_{r_i}^{r_i+D} \frac{1}{\delta D}\left(\int_{t-\delta D}^{t}s_i(t') dt' \right) dt = \int_{r_i}^{r_i+(1-\delta)D} s_i(t') \left(\int_{t'}^{t'+\delta D} \frac{1}{\delta D} dt\right) dt' = w_i .
\end{equation*}
\end{proof}
\CONVclai*
\begin{proof}
The proof only uses Jensen's inequality in the second line and the statement can be calculated as follows.
\begin{align*}
    \int_{0}^{\maxTime} \left(s^{(\delta)}(t)\right)^\alpha dt =& \int_{0}^{\maxTime} \left(\frac{1}{\delta D} \int_{t-\delta D}^{t} s(t') dt'\right)^\alpha dt\\
     \leq& \int_{0}^{\maxTime} \frac{1}{\delta D} \left(\int_{t-\delta D}^{t} (s(t'))^\alpha dt'\right) dt\\
     =& \int_{0}^{\maxTime} (s(t'))^\alpha \left(\int_{t'}^{t'+\delta D} \frac{1}{\delta D} dt\right) dt'\\
     =& \int_{0}^{\maxTime} \left(s(t)\right)^\alpha dt
\end{align*}
\end{proof}
\CONVclaim*
\begin{proof}
We have that 
\begin{align*}
    s_i^{(\delta)}(t) &=\frac{1}{\delta D}\int_{t-\delta D}^t s_i(t')dt' \leq \frac{1}{\delta D}\int_{r_i}^{r_i+D} s_i(t')dt'=\frac{w_i}{\delta D}= \frac{s_i^{\mathrm{AVR}}(t)}{\delta} .
\end{align*}
\end{proof}

\section{\robustalg for general deadlines}
\label{sec:general_deadlines}
In this section, we discuss generalizations of our techniques to general deadlines. Recall that an instance with general deadlines is defined by a set $\mathcal{J}$ of jobs $J_j = (r_j, d_j, w_j)$, where $r_j$ is the time the job
becomes available, $d_j$ is the deadline by which it
must be completed, and $w_j$ is the work to be completed. For $\delta >0$, we use the notation $\mathcal{J}^\delta$ to denote the instance obtained from $\mathcal{J}$ by shrinking the duration of each job by a factor $(1-\delta)$. That is, for each job $(r_j, d_j , w_j) \in \mathcal{J}$, $\mathcal{J}^\delta$ contains the job $(r_j, r_j + (1-\delta)(d_j - r_j), w_j)$. 

Our main result in this section  generalizes $\robustalg$ to general deadlines.
\begin{thm}
    For any $\delta > 0$, given an online algorithm for general deadlines that produces a schedule for $\mathcal{J}^\delta$ of cost $C$, we can  compute online a schedule for $\mathcal{J}$ of cost at most
    \begin{align*}
        \min\left\{ \left(\frac{1}{1-\delta}\right)^{\alpha-1} C, (2\alpha/\delta^2)^\alpha/2\cdot \OPT\right\}\,,
    \end{align*}
    where $\OPT$ denotes the cost of an optimal schedule of $\mathcal{J}$.
    \label{thm:gen_robustness}
\end{thm}

Since it is easy to design a consistent algorithm by just blindly following the prediction, we have the following corollary.

\begin{cor}
    There exists a learning augmented online algorithm for the General Speed Scaling problem, parameterized by $\varepsilon>0$,  with the following guarantees:
    \begin{itemize}
        \item \emph{Consistency:} If the prediction is accurate, then the cost of the returned schedule is at most $(1+\varepsilon) \opt$\,.
        \item \emph{Robustness:} Irrespective of the prediction, the cost of the returned schedule is at most $O(\alpha^3/\varepsilon^2)^\alpha \cdot \OPT$.
    \end{itemize}
\end{cor}
\begin{proof}[Proof of Corollary]
    Consider the algorithm that blindly follows the prediction to do an optimal schedule of $\mathcal{J}^\delta$ when in the consistent case.  That is, given the prediction of $\mathcal{J}$, it schedules all jobs that agrees with the prediction according to the optimal schedule of the predicted $\mathcal{J}^\delta$; the workload of the remaining jobs $j$ that were wrongly predicted is scheduled uniformly during their duration from release time $r_j$ to deadline $d_j$.  In the consistent case, when the prediction is accurate,  the cost of the computed schedule  equals thus the cost $\opt(J^\delta)$ of an optimal schedule of $J^\delta$. Furthermore, we have by Lemma~\ref{lms:shrink} 
    \begin{align*}
        \OPT(\mathcal{J}^\delta) \leq \left(\frac{1}{1-\delta}\right)^{\alpha-1} \OPT\,,
    \end{align*}
    where $\OPT$ denotes the cost of an optimal schedule to $\mathcal{J}$. Applying Theorem~\ref{thm:gen_robustness} on this algorithm we thus obtain an algorithm that is also robust. Specifically, we obtain an algorithm with the following guarantees:
    \begin{itemize}
        \item If prediction is accurate, then the computed  schedule has cost at most $\left(\frac{1}{1-\delta}\right)^{2(\alpha-1)}\cdot \OPT$.
        \item The cost of the computed schedule is always at most $(2\alpha/\delta^2)^\alpha/2 \cdot \OPT$.
    \end{itemize}
    
    The corollary thus follows by selecting $\delta = \Theta(\varepsilon/\alpha)$ so that $1/(1-\delta)^{2(\alpha-1)} = 1+\varepsilon$.

\end{proof}
We remark that one can also define ``smooth''  algorithms for general deadlines as we did in the uniform case. However, the prediction model and the measure of error quickly get complex and notation heavy. Indeed, our main motivation for studying the Uniform Speed Scaling problem is that it is a clean but still relevant version that allows for a natural prediction model. 

We proceed by proving the main theorem of this section, Theorem~\ref{thm:gen_robustness}.

\paragraph{The procedure \genrobustalg.} We describe the procedure  $\genrobustalg$ that generalizes $\robustalg$ to general deadlines. Its analysis then implies Theorem~\ref{thm:gen_robustness}.
Let $\mathcal{A}$ denote the online algorithm of Theorem~\ref{thm:gen_robustness} that produces a schedule of $\mathcal{J}^\delta$ of cost $C$. 
To simplify the description of \genrobustalg, we fix $\Delta > 0$ and assume that the schedule $s$ output by  $\mathcal{A}$ only changes at times that are multiples of $\Delta$.  This is without loss of generality as we can let $\Delta$ tend to $0$. To simplify our calculations, we further assume that $\delta (d_j- r_j)/\Delta$ evaluates to an integer for all jobs $(r_j, d_j, w_j) \in \mathcal{J}$.

The time line is thus partitioned into time intervals of length $\Delta$ so that in each time interval either no job is processed by $s$ or exactly one job is processed at constant speed by $s$. We denote by $s(t)$ the speed at which $s$ processes the job $j(t)$ during the $t$:th time interval, where we let $s(t)=0$ and $j(t) = \bot$ if no job was processed by $s$ (during this time interval).

To describe the schedule computed by \genrobustalg, we further divide each time interval into a \emph{base} part of length $(1-\delta) \Delta$ and an \emph{auxiliary} part of length $\delta \Delta$. In the $t$:th time interval, \genrobustalg schedules job $j(t)$ at a certain speed $s^\text{base}(t)$ during the base part, and a subset $\mathcal{J}(t) \subseteq \mathcal{J}$ of the jobs  is scheduled during the auxiliary part, each $i\in J(t)$ at a speed $s^\text{aux}_i(t)$. These quantities are computed by \genrobustalg online at the start of the $t$:th time interval as follows:
\begin{itemize}
    \item Let $s^\text{aux}(t) = \sum_{i\in \mathcal{J}(t)} s^\text{aux}_i(t)$ be the current speed of the auxiliary part and let $D_{j(t)} = d_{j(t)} - r_{j(t)}$ be the duration of job $j(t)$.
    \item If $s(t)/(1-\delta) \leq s^\text{aux}(t)$, then set  $s^\text{base}(t) = s(t)/(1-\delta)$. 
    \item Otherwise, set $s^\text{base}(t)$ so that 
    \begin{align}
     (1-\delta) \Delta s^\text{base}(t) + \left(s^\text{base}(t) - s^\text{aux}(t)\right)\delta^2 D_{j(t)} = s(t) \Delta 
% s^\text{base}(t) =  s^\text{aux}(t) +        \frac{\left(\frac{s(t)}{1-\delta} - s^\text{base}(t)\right)\Delta}{\delta D_{j(t)}}
\label{eq:genrobust_magic}
    \end{align}
    and add $j(t)$ to $J(t), J(t+1), \ldots, J(t+\delta D_{j(t)}/\Delta-1)$ with all auxiliary speeds $s^\text{aux}_{j(t)}(t), s^\text{aux}_{j(t)}(t+1), \ldots , s^\text{aux}_{j(t)}(t+ \delta D_{j(t)}/\Delta-1)$ set to $s^\text{base}(t)- s^\text{aux}(t)$.
   %  \begin{align*}
   %      \frac{\left(\frac{s(t)}{1-\delta} - s^\text{base}(t)\right)\Delta}{\delta  D_{j(t)}}\,.
   %  \end{align*}
\end{itemize}
This completes the formal description of \genrobustalg. Before proceeding to its analysis, which implies Theorem~\ref{thm:gen_robustness}, we explain the example depicted in Figure~\ref{fig:genrobust}.
% We first 
% 
% Note that $\left(\frac{s(t)}{1-\delta} - s^\text{base}(t)\right)\Delta$ equals the amount of work that $s$ processed of $j(t)$ during time interval $t$ that will \emph{not} be processed by $\genrobustalg$ during the base part and that is instead uniformly assigned to the auxiliary parts of the next $D_{j(t)}/\Delta$ time steps. In particular,  in the case when $s(t)/(1-\delta) > s^\text{aux}(t)$,
% \begin{align*}
%     (1-\delta) s^\text{base}(t) \Delta + (D_{j(t)}/\Delta)\cdot  \delta \Delta \cdot   
%         \frac{\left(\frac{s(t)}{1-\delta} - s^\text{base}(t)\right)\Delta}{\delta  D_{j(t)}}& = 
% \end{align*}
 Schedule $s$, illustrated on the left, schedules a blue, red, and green job during the first, second, and third time interval, respectively. We have that $\delta/\Delta$ times the duration of the blue job and the red job are $3$ and $4$, respectively. \genrobustalg now produces the schedule on the right where the auxiliary parts are indicated by the horizontal stripes. When the the blue job is scheduled it is partitioned  among the base part of the first interval and evenly among the auxiliary parts of the first, second and third intervals so that the speed at the first interval is the same in the base part and auxiliary part. Similarly, when the red job is scheduled, \genrobustalg splits it among the base part of the second interval and evenly among the auxiliary part of the second, third, fourth and fifth intervals so that the speed during the base part equals the speed at the auxiliary part during the second interval. Finally, the green job is processed at a small speed and is thus only scheduled in the base part of the third interval (with a speed increased by a factor $1/(1-\delta)$).
\begin{figure}
    \centering
    
    \tikzset{>=latex}
    \begin{tikzpicture}
    
   %%%%% INPUT SCHEDULE %%%%%% 
        %%%% AXIS
        \node at (3,4) {Schedule by $\mathcal{A}$};
        \draw (0,0) edge[thick, ->] (5.45,0);
        \draw(1.0, -0.1)  edge node[below=0.1cm] {\small $\Delta$} (1.0, 0.1);
        \draw(2.0, -0.1)  edge node[below=0.1cm] {\small $2\Delta$} (2.0, 0.1);
        \draw(3.0, -0.1)  edge node[below=0.1cm] {\small $3\Delta$} (3.0, 0.1);
        \draw(4.0, -0.1)  edge node[below=0.1cm] {\small $4\Delta$} (4.0, 0.1);
        \draw(5.0, -0.1)  edge node[below=0.1cm] {\small $5\Delta$} (5.0, 0.1);
        \node at (5.8, 0) {time};
        \draw (0,0) edge[thick, ->] (0,3);
        \node at (0, 3.3) {speed};
        % \draw (-0.1, 0.75) edge node[left=0.1cm]{$1$} (0.1, 0.75);
        % \draw (-0.1, 1.5) edge node[left=0.1cm]{$2$}(0.1, 1.5);
        % \draw (-0.1, 2.25) edge node[left=0.1cm]{$3$}(0.1, 2.25);
    
        %%% JOBS    
        \draw[fill=blue!20!white] (0.0, 0) rectangle (1.0, 1);
        \draw[fill=red!20!white] (1.0, 0) rectangle (2.0, 2);
        \draw[fill=green!20!white] (2.0, 0) rectangle (3.0, 0.25);
%        \node at (0.375, -0.25) {0};
        
        \begin{scope}[xshift=7cm]
   %%%%% OUTPUT SCHEDULE %%%%%% 
        \node at (3,4) {Schedule by $\genrobustalg$};
        %%%% AXIS
        \draw (0,0) edge[thick, ->] (5.45,0);
        \draw(1.0, -0.1)  edge node[below=0.1cm] {\small $\Delta$} (1.0, 0.1);
        \draw(2.0, -0.1)  edge node[below=0.1cm] {\small $2\Delta$} (2.0, 0.1);
        \draw(3.0, -0.1)  edge node[below=0.1cm] {\small $3\Delta$} (3.0, 0.1);
        \draw(4.0, -0.1)  edge node[below=0.1cm] {\small $4\Delta$} (4.0, 0.1);
        \draw(5.0, -0.1)  edge node[below=0.1cm] {\small $5\Delta$} (5.0, 0.1);
        \node at (5.8, 0) {time};
        \draw (0,0) edge[thick, ->] (0,3);
        \node at (0, 3.3) {speed};
        % \draw (-0.1, 0.75) edge node[left=0.1cm]{$1$} (0.1, 0.75);
        % \draw (-0.1, 1.5) edge node[left=0.1cm]{$2$}(0.1, 1.5);
        % \draw (-0.1, 2.25) edge node[left=0.1cm]{$3$}(0.1, 2.25);
    
        %%% JOBS    
        \draw[fill=blue!20!white] (0.0, 0) rectangle (0.8, 0.7);
        \draw[fill=blue!20!white] (0.8, 0) rectangle (1.0, 0.7);
        \draw[pattern=horizontal lines] %, pattern color=red!20!white] 
                    (0.8, 0) rectangle (1.0, 0.7);
        \draw[fill=red!20!white] (1.0, 0) rectangle (1.8, 1.5);
        \draw[fill=blue!20!white] (1.8, 0) rectangle (2.0, 0.7);
        \draw[fill=red!20!white] (1.8, 0.7) rectangle (2.0, 1.5);
        \draw[pattern=horizontal lines] %, pattern color=red!20!white] 
                    (1.8, 0) rectangle (2.0, 1.5);
        \draw[fill=green!20!white] (2.0, 0) rectangle (2.8, 0.30);
        
        \draw[fill=blue!20!white] (2.8, 0) rectangle (3.0, 0.7);
        \draw[fill=red!20!white] (2.8, 0.7) rectangle (3.0, 1.5);
        \draw[pattern=horizontal lines] %, pattern color=red!20!white] 
                    (2.8, 0) rectangle (3.0, 1.5);
        
        \draw[fill=red!20!white] (3.8, 0.0) rectangle (4.0, 0.8);
        \draw[pattern=horizontal lines] %, pattern color=red!20!white] 
                    (3.8, 0) rectangle (4.0, 0.8);
        \draw[fill=red!20!white] (4.8, 0.0) rectangle (5.0, 0.8);
        \draw[pattern=horizontal lines] %, pattern color=red!20!white] 
                    (4.8, 0) rectangle (5.0, 0.8);
                    
%        \node at (0.375, -0.25) {0};
        \end{scope}
    \end{tikzpicture}
    \caption{Given the schedule on the left, \genrobustalg produces the schedule on the right.}
    \label{fig:genrobust}
\end{figure}
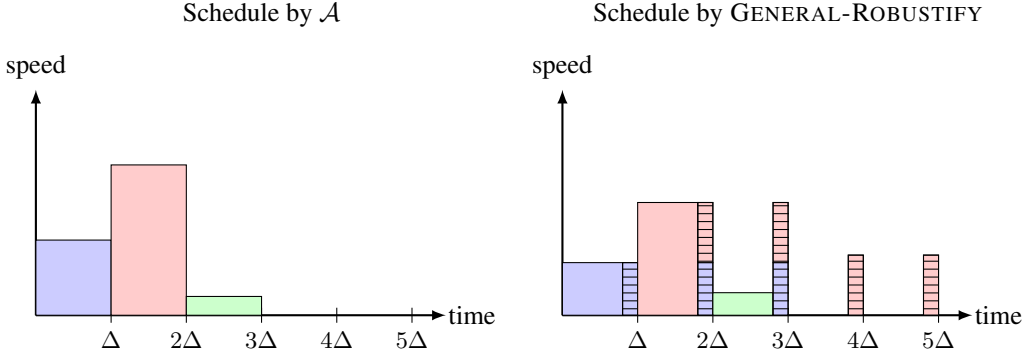

\paragraph{Analysis.}  We show that \genrobustalg satisfies the guarantees stipulated by Theorem~\ref{thm:gen_robustness}. We first argue that \genrobustalg produces a feasible schedule to $\mathcal{J}$. During the $t$:th interval, the schedule $s$ computed by $\mathcal{A}$ processes $\Delta \cdot s(t)$ work of job $j(t)$. We argue that \genrobustalg processes the same amount of work from this time interval. At the time when this interval is considered by \genrobustalg, there are two cases:
\begin{itemize}
    \item If $s(t)/(1-\delta) \leq s^\text{aux}(t)$ then $s^\text{base}(t) = s(t)/(1-\delta)$ so \genrobustalg processes $(1-\delta)\Delta s(t)/(1-\delta) = s(t) \Delta$ work of $j(t)$ during the base part of the $t$:th time interval.
    \item Otherwise, we have that \genrobustalg processes $(1-\delta)\Delta s^\text{base}(t)$ of $j(t)$ during the base part of the $t$:th time interval and $\delta \Delta \left(s^\text{base}(t) - s^\text{aux}(t)\right)$ during the auxiliary part of each of the $\delta D_{j(t)}/\Delta$ time intervals $t, t+1, \ldots, t+ \delta D_{j(t)}/\Delta - 1$. By the selection~\eqref{eq:genrobust_magic}, it thus follows that \genrobustalg processes all work $s(t)\Delta $ from this time interval. in this case as well.
\end{itemize}
The schedule of \genrobustalg thus completely processes every job. Furthermore, since each job is delayed at most  $\delta D_{j(t)}$ time steps we have that it is a feasible schedule to $\mathcal{J}$ since we started with a schedule for $\mathcal{J}^\delta$, which completes each job $j$ by time $r_j + (1-\delta ) D_j$. It remains to prove the robustness and soundness guarantees of Theorem~\ref{thm:gen_robustness}
\begin{lem}[Robustness]
   \genrobustalg computes a schedule of cost at most $(2\alpha/\delta^2)^\alpha/2\cdot \OPT$. 
\end{lem}
\begin{proof}
    By the definition of the algorithm we have, for each time interval, that the speed of the base part is at most the speed of the auxiliary part. Letting $s^\text{base}(t)$ and $s^\text{aux}(t)$ denote the speed of the base and auxiliary part of the $t$:th time interval, we thus have
    \begin{align*}
        \sum_{t} \left((1-\delta) s^\text{base}(t)^\alpha + \delta s^\text{aux}(t)^\alpha \right) \leq \sum_t s^\text{aux}(t)^\alpha\,.
    \end{align*}
    Now we have that the part of a job $j$ that is processed during the auxiliary part of a time interval has been uniformly assigned to at least $\delta^2 D_j$ time steps. It follows that the speed at any auxiliary time interval is at most $1/\delta^2$ times the speed at that time of the \textsc{Average Rate} heuristic (AVR). The lemma now follows since that heuristic is known \cite{DBLP:conf/focs/YaoDS95} to have  competitive ratio at most  $(2\alpha)^\alpha/2$.
\end{proof}

\begin{lem}[Consistency] 
   \genrobustalg computes a schedule of cost at most $\left(\frac{1}{1-\delta}\right)^{\alpha-1}\cdot C$ where $C$ denotes the cost of the schedule $s$ computed by $\mathcal{A}$. 
\end{lem}
\begin{proof}
    For $t\geq 0$, let $h^{(t)}$  be the schedule that processes the workload during the first $t$ time intervals as in the schedule computed by \genrobustalg, and the workload of the remaining time intervals is processed during the base part of that time interval by increasing the speed by a factor $1/(1-\delta)$. Hence, $h^{(0)}$ is the schedule that processes  the workload of all time intervals during the base part at a speed up of $1/(1-\delta)$, and $h^{(\infty)}$ equals the schedule produced by $\genrobustalg$. By definition, the cost of $h^{(0)}$ equals $\left(\frac{1}{1-\delta}\right)^\alpha (1-\delta)\cdot C$ and so the lemma follows by observing that for every $t\geq 1$ the cost of $h^{(t)}$ is at most the cost of $h^{(t-1)}$. To see this consider the two cases of \genrobustalg when considering the $t$:th time interval:
    \begin{itemize}
        \item If $s(t)/(1-\delta) \leq s^\text{aux}(t)$ then \genrobustalg processes all the workload during the base part at a speed of $s^\text{base}(t) = s(t)/(1-\delta)$. Hence, in this case,  the schedules $h^{(t)}$ and $h^{(t-1)}$ processes the workload of the $t$:th time interval identically and so they have equal costs. 
        
        \item Otherwise, \genrobustalg partitions the workload of the $t$:th time interval among the base part of the $t$:th interval and $\delta D_{j(t)}/\Delta$ many auxiliary parts so that the speed at each of these parts is strictly less than $s(t)/(1-\delta)$. Hence, since $h^{(t)}$ processes the workload of the $t$:th time interval at a lower speed than $h^{(t-1)}$ we have that its cost is strictly lower if $\alpha > 1$ (and the cost is equal if $\alpha=1$).  
    \end{itemize}
\end{proof}

\section{Additional Experiments}
\label{sec:Additional Experiments}
In this section we further explore the performance of LAS algorithm for different values of the parameter $\alpha$. We conduct experiments on the login requests of \textit{BrightKite} using the same experimental setup used in Section \ref{sec:Experimental analysis}. The results are summarized in Table \ref{table:table with alpha parameters}. In every column the average competitive ratios of each algorithm for a fixed $\alpha$ are presented. We note that, as expected, higher values of $\alpha$ penalize heavily wrong decisions deteriorating the competitive ratios of all algorithms. Nevertheless, \textsc{LAS} algorithm consistently outperforms \textsc{AVR} and \textsc{OA} for all different values of $\alpha$.

\begin{table}[H]
\centering
\caption{Real dataset results with different $\alpha$ values}
\vskip 0.1in
\label{table:table with alpha parameters}
{\small
\begin{tabular}{lccccr}
\toprule
Algorithm & $\alpha =3$ & $\alpha = 6$ & $\alpha = 9$ & $\alpha = 12$ \\
\midrule
AVR     & 1.365 & 2.942 & 7.481 & 21.029 \\
OA     & 1.245     & 2.211     & 4.513 & 9.938  \\
LAS, $\epsilon = 0.8$    & 1.113& 1.576& 2.806 & 7.204 \\
LAS, $\epsilon = 0.01$    & 1.116& 1.598& 2.918 & 8.055 \\
\bottomrule
\end{tabular}
}
\caption*{The timeline was discretized in chunks of ten minutes and $D$ was set to 20.}
\vskip -0.2in
\end{table}

}{}

\end{document}

%%% Local Variables:
%%% mode: latex
%%% TeX-master: t
%%% End: